\documentclass[10pt,twocolumn,twoside]{IEEEtran} 
               
\IEEEoverridecommandlockouts 

\newcommand{\diffs}[3]{\frac{\partial^2 #1}{
\ifx#2#3 
\partial #2^2
\else
\partial #2 \partial #3
\fi
}}

\usepackage{amsmath} 
\usepackage{amssymb} 
\usepackage{amsthm}
\usepackage{amsfonts}
\usepackage{mathtools}

\usepackage{paralist}
\usepackage{lipsum}

\newtheorem*{problem*}{Problem}
\newtheorem{prop}{Proposition}

\newtheorem{result}{Result}

\theoremstyle{definition}

\theoremstyle{remark}

\newtheorem{theorem}{Theorem}[section]
\newtheorem{remark}[theorem]{Remark}

\usepackage{verbatim}

\usepackage{color}
\usepackage{subfig}
\usepackage{graphicx}
\graphicspath{{./figures/}}

\usepackage{caption}

\usepackage{listings}

\usepackage[ruled,vlined,linesnumbered]{algorithm2e}
\usepackage[noend]{algorithmic}

\usepackage[us]{datetime}
\usepackage{caption}
\newcommand\red[1]{{\textcolor{red}{#1}}}
\newcommand\blue[1]{{\textcolor{blue}{#1}}}

\newcommand{\rev}[1]{} 

\usepackage[usenames,dvipsnames,table]{xcolor}

\usepackage{hyperref} 
\hypersetup{
 colorlinks,
 citecolor=black,
 filecolor=black,
 linkcolor=black,
 urlcolor=black,
 pdfauthor={},
 pdfsubject={},
 pdftitle={}
}

\usepackage{todonotes}

\usepackage{graphicx}

\usepackage{subfig}

\usepackage{latexsym}
\usepackage{color}

\usepackage{cite}

\usepackage{caption}

\usepackage{soul}
\usepackage{bm}

\title{Distributed Estimation of State and Parameters in\\ Multi-Agent Cooperative Load Manipulation}

\author{Antonio Franchi, \textit{Senior Member}, \textit{IEEE}, \and Antonio Petitti, \and Alessandro Rizzo, \textit{Senior Member}, \textit{IEEE}
\thanks{A.~Franchi is with CNRS, LAAS, 7 avenue du colonel Roche, F-31400 Toulouse, France and Univ de Toulouse, LAAS, F-31400 Toulouse, France {\tt \scriptsize \href{mailto:afranchi@laas.fr}{afranchi@laas.fr}}}
\thanks{A.~Petitti is with the Institute of Intelligent Industrial Technologies and Systems for Advanced Manufacturing, National Research Council (STIIMA-CNR), 70126 Bari, Italy, {\tt \scriptsize\href{mailto:antonio.petitti@stiima.cnr.it}{antonio.petitti@stiima.cnr.it}}}
\thanks{A.~Rizzo is with the Dipartimento di Elettronica e Telecomunicazioni, Politecnico di Torino, Corso Duca degli Abruzzi 24, 10129 Torino, Italy {\tt \scriptsize \href{mailto:alessandro.rizzo@polito.it}{alessandro.rizzo@polito.it}}}
   \thanks{This research was partially supported by the ANR, Project ANR-17-CE33-0007 MuRoPhen, by the Siebel Energy Institute, and by Compagnia di San Paolo.}
}

\begin{document}

\maketitle

\begin{abstract}
We present two distributed methods for the estimation of the kinematic parameters, the dynamic parameters, and the kinematic state of an unknown planar body manipulated by a decentralized multi-agent system\rev{ground (planar) robots}. The proposed approaches rely on the rigid body kinematics and dynamics, on nonlinear observation theory, and on consensus algorithms. The only three requirements are that each agent can exert a 2D wrench on the load, it can measure the velocity of its contact point, and that the communication graph is connected. Both theoretical nonlinear observability analysis and convergence proofs are provided.
The first method assumes constant parameters while the second one can deal with time-varying parameters and can be applied in parallel to any task-oriented control law. For the cases in which a control law is not provided, we propose a distributed and safe control strategy satisfying the observability condition. The effectiveness and robustness  of the estimation strategy is showcased by means of realistic MonteCarlo simulations. 
\end{abstract}

\ifdefined\INTROLIPSUM

\section{Introduction}\label{sec:intro}

\lipsum[1-7]

\else

\section{Introduction}\label{sec:intro}

In this paper, we propose what we believe is the first fully-distributed method for the estimation of all the quantities and parameters needed by a \rev{ground (planar)}planar robotic multi-agent system to collectively manipulate an unknown load. In particular, the proposed algorithm provides the estimation of the kinematic parameters (equivalent to the grasp matrix), the dynamic parameters (relative position of the center of mass, mass, and rotational inertia) and the kinematic state of the load (velocity of the center of mass and rotational rate). 

Most of the works on cooperative manipulation in the literature assume the \textit{a priori} knowledge of the inertial parameters of the load, even though this assumption does not always hold in real-world scenarios~\cite{1989-KimZhe,1992-SchCan,1991-WalFreMar,2002-SzePluBid,2013-SieDerHir}. Collective manipulation tasks would benefit from the implementation of on-line estimation strategies of inertial parameters of unknown loads for at least two reasons: first, existing control strategies, such as force control and pose estimation, could be effectively applied with satisfactory performance and a reduced control effort. Second, time-varying loads could be effectively manipulated, toward the implementation of adaptive or event-driven control strategies in uncertain environments.
 Furthermore, similarly to other applications in multi-agent systems,
a distributed and decentralized implementation of such estimation strategies would provide robustness and scalability.
Research on the estimation of inertial parameters is at its early stage, and main limitations of the existing approaches reside in their centralization and in the use of absolute position and acceleration measurements, which are hard and costly to achieve, especially if accurate and noise-free information is needed. Moreover, centralized strategies are notoriously poorly scalable and not robust, due to the existence of single points of failure~\cite{2005-YonAriTsu,2008-KubKroWah,2013-ErhHir,2012-MarKarHu_Kra}.

In this paper we propose two algorithms that have instead the following characteristics:
{\it (i)}~there is no central processing unit;
{\it (ii)}~each agent is only able to exchange information with its neighbors in the communication network;
{\it (iii)}~the network, modeled as an undirected graph, can have any topology, provided that it is connected;
{\it (iv)}~each agent is able only to perform local sensing and computation; and 
{\it (v)}~the amount of memory and number of computations per step needed by each local instance of the algorithm do not depend on the number of agents but only on the number of communication neighbors.
The only assumptions needed are that each agent is\rev{endowed with a planar manipulator that is} able to apply a wrench to the load at its contact point and to measure the velocity of such contact point. Any other measurement (such as, e.g., position, distance, acceleration, and gyro measurements) is not available to the agents. Finally, nothing is known about the manipulated load.
The approaches are totally distributed, and rely on the geometry of the rigid body kinematics, the rigid body dynamics, on nonlinear observation theory, and on consensus strategies.

\emph{Related works:} In~\cite{2015-WanSch}, a decentralized motion control approach based on force consensus, which does not rely on explicit communication among the agents, is proposed. However, the result is achieved at the expense of assuming the presence of a leader agent that steers the load and on an even number of agents arranged in a particular shape called by the authors `centrosymmetric'. On the contrary, here we assume that agents can actually communicate, yet we do not assume the presence of a leader and we allow for any unknown spatial arrangement of the contact points. 
A communication-less adaptive motion control strategy is proposed in \cite{2018-CulSch}, under the assumption of a centralized measurement of the center of mass velocity and the load angular velocity. In the methods proposed here, no centralized measurement is needed. Furthermore, differently from~\cite{2015-WanSch,2018-CulSch} our two methods let each agent estimate all the parameters of the problem, thus paving the way for the utilization of any control task (e.g., motion control, force control, etc.) on top of the proposed estimation algorithm.
The authors in~\cite{2018-MarPie} show how a `communication-less' parameter estimation can be achieved by adding some stronger assumptions. The method assumes initial parallel execution of synchronized control sequences by all the agents, which requires a form of centralization, prevents the  simultaneous estimation of the parameters while performing the control task, and is not suited to estimate time-varying  parameters.  In our two methods, a-priori synchronization is not required, and the second method (see Sec.~\ref{sec:TimeVarInertParams}) can estimate time-varying parameters while performing any control task. The method in~\cite{2018-MarPie} assumes also that the robot are localized in position and orientation on a common frame and that they have enough strength to perform a hybrid position/force control and to lift the load from the ground, exploiting the gravity to estimate the mass. Such assumptions are absent in our setting. In fact, our main contribution is to demonstrate that if communication is available then it is possible to solve the estimation in a fully distributed way with minimal sensing. 

\fi

\section{Model and Problem Statement}\label{sec:probStat}
\begin{figure}[t]
\centering
\includegraphics[width=0.35\textwidth]{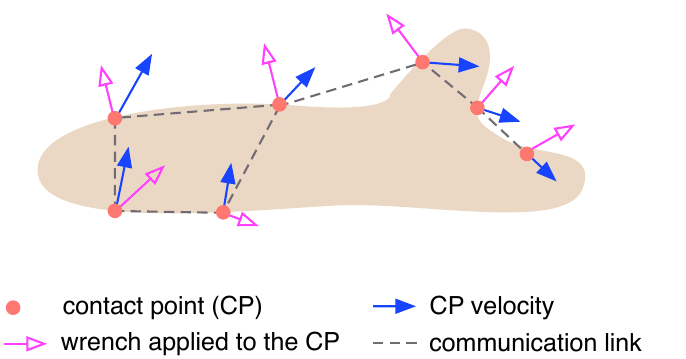}
\caption{\rev{\red{[AR: ATTENZIONE TYPO NELLA FIGURA, TORQUE NOT TORQE]}\blue{AP:AntonioF., hai tu il sorgente di questa immagine?}}
Conceptual representation of a group of $n$ agents manipulating an object on a plane. The orange dots represent the contact points of each agent, each blue arrow is the velocity of the contact point measured by each agent and each magenta hollow arrow is the wrench (force and torque) applied by each agent at the corresponding contact point. Dashed lines represent the communication links between agents, which all together constitute the communication graph.\rev{Top view of a team of five mobile manipulators performing a planar manipulation task. Each agent can exert a force and torque on the object by means of a planar manipulator (only force is displayed in the picture), and can only measure the velocity of its contact point.}}
\label{fig:ProbStat}
\end{figure}

In this section, we formally define the problem of distributively estimating all the parameters and the time-varying quantities 
that are present in the mechanical model of a team of $n$ planar robotic agents that cooperatively manipulate an unknown planar load, as illustrated in Fig.~\ref{fig:ProbStat}.  
We assume $n$ to be constant and known to all robots. This assumption can be easily relaxed by implementing one of the several algorithms for the distributed estimation of a graph size~\cite{2000-Tel}. 

We denote the inertial frame with $\mathcal{W}=O-\mathbf{xy}$ 
and the load body frame with $\mathcal{B}=C-\mathbf{x}_B\mathbf{y}_B$, where $C$ is the center of mass (CoM) of the unknown body, indicated with $B$. We indicate with $\mathbf{p}_C\in\mathbb{R}^2$ and $\mathbf{v}_C=\dot{\mathbf{p}}_C$ the position and velocity of $C$ expressed in $\cal W$, respectively, and with $\omega\in\mathbb{R}$ the intensity of the load angular velocity, hereafter called simply \emph{angular rate}. 
The dynamics of the manipulated load is described by the rigid body dynamical equation
\begin{equation}
\mathbf{M} \dot{\boldsymbol{\nu}} + \mathbf{g} = \mathbf{u},
\label{eq:dynamics_load}
\end{equation}
where $\boldsymbol{\nu}=(\mathbf{v}_C^\top\;\omega)^\top\in\mathbb{R}^3$ is the twist of $B$; $\mathbf{M}={\rm diag}(m,m,J)\in\mathbb{R}^{3\times 3}$ is the inertia matrix with $m>0$ and $J>0$ being the mass and the rotational inertia of the load, respectively; $\mathbf{g}\in\mathbb{R}^3$ is the wrench resulting from the environmental forces such as friction or gravitation. Without loss of generality, in this work we set $\mathbf{g} = \mathbf{0}$. This is equivalent to assume a horizontal workspace and a wheeled load, for which the friction effects are negligible~\cite{2016-FraMag}. 
 Finally, $\mathbf{u}\in\mathbb{R}^3$ denotes the external wrench applied by the agents to $B$, which will be characterized in the following.
All the previous quantities are expressed with respect to the frame $\mathcal{W}$.

Each agent $i$ contributes to the overall manipulation by exerting a wrench $\mathbf{u}_i=(\mathbf{f}_i^\top\;\tau_i)^\top\in\mathbb{R}^3$,
 expressed in $\mathcal{W}$, with $i=1 \ldots n$. The force $\mathbf{f}_i\in\mathbb{R}^2$ is applied by agent $i$ to a contact point $C_i$ of $B$,
 and $\tau_i\in \mathbb{R}$ is the intensity of the torque applied about the normal direction to the plane $\mathbf{xy}$.
We assume that contact points do not overlap, i.e., $C_i\neq C_j$, $\forall i,j=1\ldots n$.
The total external wrench applied to~$B$ is
$\mathbf{u} = \sum_{i=1}^n \mathbf{G}_i \mathbf{u}_{i} = \mathbf{G} \bar{\mathbf{u}}$, %
where $\mathbf{G}_i\in\mathbb{R}^{3\times 3}$ is the partial grasp matrix, $\mathbf{G}\in\mathbb{R}^{3\times 3 n}$ is the grasp matrix, and $\bar{\mathbf{u}} = \left({\mathbf{u}_1}^\top , \, \dots, \, {\mathbf{u}_n}^\top\right)^\top$ is the stacked applied wrench vector that groups the generalized contact force components transmitted through the contact points~\cite{2008-PraTri}.
The partial grasp matrix is defined as $\mathbf{G}_i = \mathbf{P}_i \bar{\mathbf{R}}_i$, where 
$\mathbf{P}_i = 
\left[\begin{smallmatrix}
\mathbf{I}_{2\times 2} & {\bf 0}_{2\times 1} \\
\left[(\mathbf{p}_{C_i} - \mathbf{p}_C)^\perp\right]^\top & 1 \\
\end{smallmatrix}\right]$
and $\bar{\mathbf{R}}_i=\mathbf{I}_{3\times 3}$, in our setting, for all $i=1 \ldots n$.
Here, $\mathbf{p}_{C_i}\in\mathbb{R}^2$ is the position of $C_i$ in $\cal W$.
The operator $(\cdot)^\perp$ rotates a vector by an angle of $\pi/2$, as is defined as 
$\mathbf{q}^{\perp}= Q \mathbf{q}=
[-q^{y}\; q^{x}]^\top$ where 
$Q=
\left[
\begin{smallmatrix}
0 & -1\\
1 & 0
\end{smallmatrix}
\right]$.
Hence, the dynamics~\eqref{eq:dynamics_load} of the manipulated load is given by %
\begin{equation}
\begin{bmatrix}
\dot{\mathbf{v}}_C\\
\dot \omega
\end{bmatrix}
=
\sum_{i=1}^{n}
\begin{bmatrix}
m^{-1}\mathbf{I}_{2\times 2} & {\bf 0}_{2\times 1} \\
J^{-1}\left[(\mathbf{p}_{C_i} - \mathbf{p}_C)^\perp\right]^\top & J^{-1} \\
\end{bmatrix}
\begin{bmatrix}
\mathbf{f}_i\\
\tau_i
\end{bmatrix}.
\label{eq:dynamics_1}
\end{equation}

Let $\mathbf{p}_{G}\in\mathbb{R}^2$ represent the position of the geometric center $G$ of the contact points expressed in $\mathcal{W}$, i.e.,
$\mathbf{p}_{G} = \frac{1}{n}\sum_{i=1}^n \mathbf{p}_{C_i}$.
We compactly define $\mathbf{z}_i = \mathbf{p}_{C_i} - \mathbf{p}_G$ and $\mathbf{z}_C = \mathbf{p}_G - \mathbf{p}_C$.
Hence, substituting $\mathbf{p}_{C_i} - \mathbf{p}_{C} = \mathbf{z}_i +\mathbf{z}_C $ in~\eqref{eq:dynamics_1} we obtain
\begin{equation}
\begin{bmatrix}
\dot{\mathbf{v}}_C\\
\dot \omega
\end{bmatrix}
=
\sum_{i=1}^{n}
\begin{bmatrix}
m^{-1}\mathbf{I}_{2\times 2} & {\bf 0}_{2\times 1} \\
J^{-1}{\mathbf{z}_i^\perp}^\top & J^{-1} \\
\end{bmatrix}
\begin{bmatrix}
\mathbf{f}_i\\
\tau_i
\end{bmatrix} +
\begin{bmatrix}
{\bf 0}_{2\times 1}\\
J^{-1}{\mathbf{z}_C^\perp}^\top \\
\end{bmatrix}
\mathbf{f}_i.
\label{eq:dynamics_rewritten}
\end{equation}

Inspecting the dynamics~\eqref{eq:dynamics_rewritten}, it is possible to see~\cite{2016b-PetFraDipRiz} that in order to effectively manipulate the load by controlling its velocity $\mathbf{v}_C$ and angular rate $\omega$, it is of fundamental importance that each agent $i$ has an estimate of the constant parameters $m$ and $J$, the time-varying vectors $\mathbf{z}_i(t)$ and $\mathbf{z}_C(t)$, and the quantities to be controlled, i.e., $\mathbf{v}_C(t)$ and $\omega(t)$. 

Finally, the communication network among agents is modeled as a connected undirected graph $\mathcal{G}$, whose both node and link set are assumed to be time-invariant. The set $\mathcal{N}_i$ denotes the se of one-hop neighbors of agent $i$ in the communication graph, while $\mathcal{A}$ denotes the graph adjacency matrix. 

The problem under investigation is formally stated next. 
\begin{problem*}[Distributed Estimation in Multi-Agent Cooperative Manipulation]\label{prob:main}
Given $n$ agents communicating through an
 ad-hoc network
 and manipulating an unkown load $B$; assume that each agent~$i$ can only
\begin{enumerate}
\item locally measure the velocity $\mathbf{v}_{C_{i}}$ of the contact point $C_{i}$,
\item locally know the applied wrench $\mathbf{u}_{i}$ acting on $C_{i}$,
\item communicate with its one-hop neighbors in the communication network. 
\end{enumerate}
Design a \emph{fully-distributed algorithm} such that each agent~$i$ is able to estimate the following six quantities: 
\begin{enumerate}
\item the (constant) mass $m$ of the load,
\item the (constant) rotational inertia $J$ of the load,
\item the (time-varying) relative position $\mathbf{z}_i(t)$ of $C_i$ with respect to the geometric center $G$ of the contact-points,
\item the (time-varying) relative position $\mathbf{z}_C(t)$ of the Center of Mass (CoM) $C$ of $B$ with respect to $G$,
\item the (time-varying) velocity $\mathbf{v}_C(t)$ of $C$, and
\item the (time-varying) angular rate $\omega(t)$ of $B$.
\end{enumerate}
\end{problem*}

In this work, we consider a strict definition of distributed algorithm, such that it is highly scalable with respect to the number of agents $n$. The main requisite of such an algorithm is that the complexity of the computations performed locally by each agent (in terms both of the number of elementary operations and of size of the input/output data) has to be constant with respect to the number of agents $n$~\cite{2013e-RobFraSecBue}.

In the next sections we shall constructively prove that the Problem is solvable as long as the communication network is connected, i.e., there exists a multi-hop communication path from any agent to any other agent in the network.

\section{Overview of the two Algorithms}\label{sec:overview}

An overview of the first proposed distributed estimation algorithm is given in the scheme of Fig.~\ref{fig:AlgoOverview}. Each rectangular box in the scheme corresponds to a computation performed locally by each agent~$i$. Each circle, instead, corresponds to a consensus-like distributed algorithm that is used to compute the only five global quantities that we shall prove to be needed in the distributed estimation process. The number of these global quantities is independent from the number of agents, and they can be estimated using standard distributed algorithms. Therefore, the overall distributiveness of the approach is ensured. The convergence of the adopted distributed algorithms requires only that the overall communication graph is connected (no all-to-all communication is required). The same applies for our distributed estimation algorithm.

\begin{figure}[t]
\centering
\includegraphics[width=.95\columnwidth]{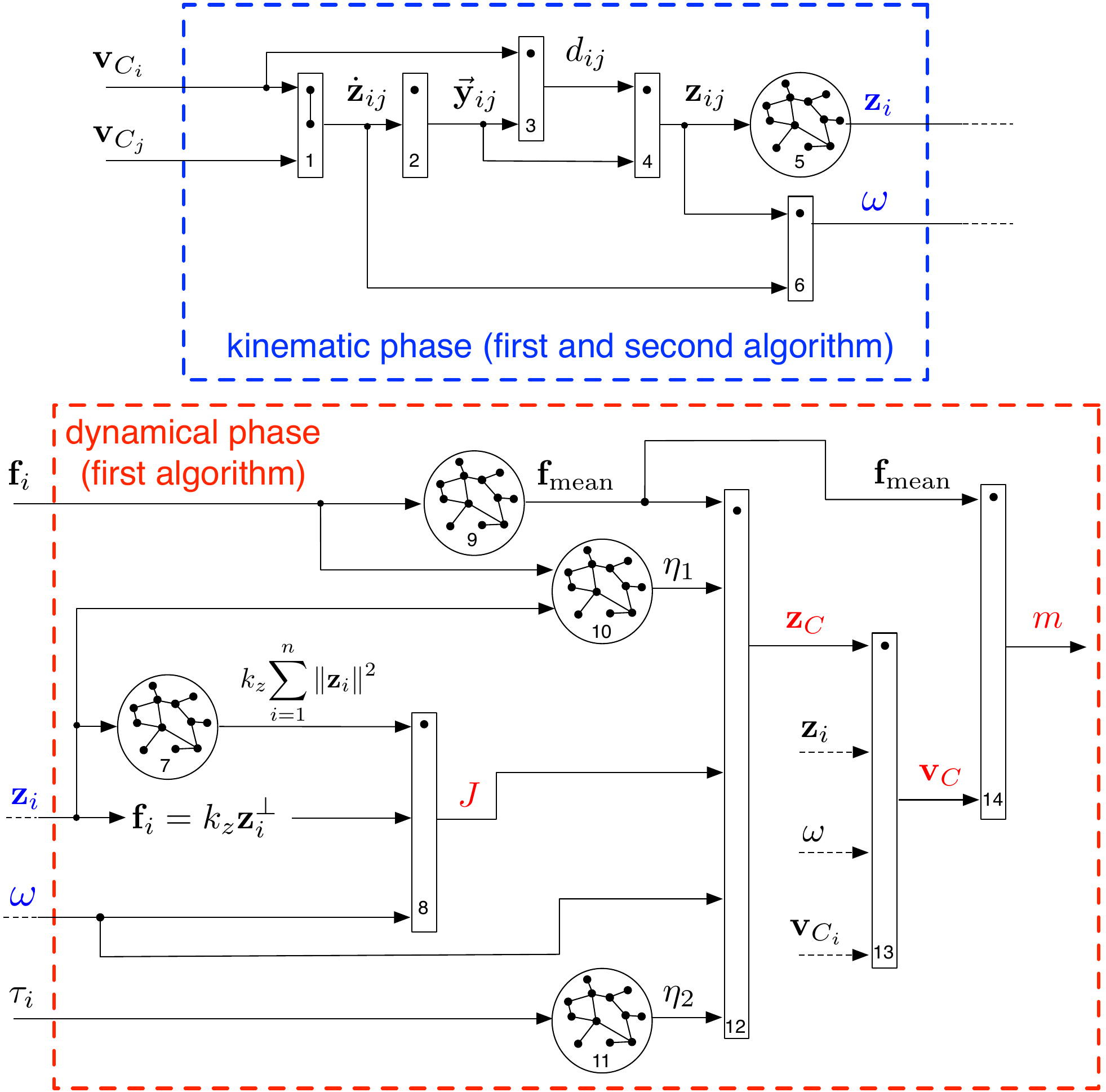}
\caption{Overview of the first proposed distributed estimation algorithm. \textbf{Top} (dashed blue box): \emph{purely kinematic} phase, where only the velocity measurements and the rigid body kinematics are used. After this phase, the estimates of the time-varying quantities $\mathbf{z}_i(t)$ and $\omega(t)$ (in blue) become available to each agent~$i$. \textbf{Bottom} (dashed red box): \emph{dynamical phase}, where also the knowledge of the wrenches and the rigid body dynamics are used. After this phase, the quantities $J$, $\mathbf{z}_C(t)$, $\mathbf{v}_C(t)$, and $m$ (in red) become available to each agent~$i$. %
}
\label{fig:AlgoOverview}
\end{figure}

To better understand the overall functioning of the algorithm, it is convenient to logically decompose its structure in a \emph{purely kinematic} phase, followed by a \emph{dynamical} one. In the former, only the rigid body kinematics constraints and the velocity measurements are used. After this phase, each agent~$i$ is able to estimate the time-varying quantities $\mathbf{z}_i(t)$ and $\omega(t)$. 
In the latter, the applied wrench and the rigid body dynamics are also used. After this phase, each agent is able to estimate the remaining quantities, i.e., $J$, $\mathbf{z}_C(t)$, $\mathbf{v}_C(t)$, and $m$. 
The two phases are described in Sections~\ref{sec:kyn_phase} and~\ref{sec:dyn_phase}, respectively.

All the estimation blocks are cascaded, hence convergence/inconsistency issues of  feedback estimation structures do not affect our scheme. 
A clarification about the convergence of our strategy is in order. Some steps of the estimation procedure are achieved through averaging consensus algorithms, which are known to converge asymptotically. Although this aspect theoretically yields infinite convergence times, an \textit{$\epsilon$-approximate global consensus}~\cite{2017-ManHad}, up to an arbitrary accuracy, can be achieved after a convenient finite time interval. \rev{, regulated, as commonly known, by the second smallest eigenvalue of the Laplacian matrix of the agents' communication graph~\cite{2011-OlsTsi}.}

The first estimation algorithm assumes constant $J$ and $m$ and represents the best choice, in terms of noise filtering, in that case. The estimation of $J$ requires a special wrench input, which prevents the use of another control algorithm in that phase.  To overcome such possible drawbacks, a second estimation algorithm is also proposed. The second algorithm is designed as a variant of the first one, and described in Sec.~\ref{sec:TimeVarInertParams}, see Fig.~\ref{fig:AlgoOverviewNew}. It can deal with changing $J$ and $m$ and does not require any particular wrench in any of its phases.
\section{Kinematic Phase}\label{sec:kyn_phase}

The objective of this phase 
is to distributively compute an estimate of the time-varying quantities $\mathbf{z}_i(t)$ and $\omega(t)$, based only on the locally measured velocities and the rigid body kinematic constraints.
The basic idea of this phase is to split the estimations of $\mathbf{z}_i(t)$ and $\omega(t)$ in two parts. The former is common to both estimations and essentially consists of the estimation of $\mathbf{z}_{ij}(t)=\mathbf{p}_{i}(t)-\mathbf{p}_{j}(t)$. This part is described in Sec.~\ref{sec:est_z_ij}. The latter comprises two separate estimators of $\mathbf{z}_i(t)$ and $\omega(t)$ and is described in Sec.~\ref{sec:est_z_i_omega}

The reason for passing through the estimation of the quantities $\mathbf{z}_{ij}$ is briefly explained in the following. In~\cite{2012-AraCarSagCal}, a distributed algorithm is proposed that allows the estimation of the centroid of the positions of a network of agents by only measuring the relative positions between pairs of communicating agents.
This algorithm can be used 
to distributively estimate $\mathbf{z}_i(t)$ if each pair of neighbors 
know the relative position of the contact points $C_i$ and $C_j$, i.e., $\mathbf{z}_{ij}(t)$. 
Nevertheless, here each agent only measures the velocity of its contact point and not its position. Our first contribution is to show that, thanks to the rigid body constraint, it is possible to estimate $\mathbf{z}_{ij}(t)$ only resorting to the measures $\mathbf{v}_{i}(t)$ and $\mathbf{v}_{j}(t)$.
Hereinafter, to enhance readability, we shall drop the time dependence of variables where such a dependence is clear from the context.

\subsection{Estimation of $\mathbf{z}_{ij}(t)$}\label{sec:est_z_ij}
\rev{\red{[AR: il problema del sistema di riferimento influenza anche questa fase? il revisore parlava da eq 16 in poi, io la metto li, se invece va messa qui la spostiamo]. }\blue{AP: beh, si. Direi che il problema dei SdR sussiste fin dall'inizio} \blue{[AF: si' vale dall'inizio.]}} The time-varying vector $\mathbf{z}_{ij}(t)$ that we want to estimate has to obey to the \emph{nonlinear} rigid body constraint
\begin{align}\label{eq:rigid_body_dist}
\mathbf{z}_{ij}^\top\mathbf{z}_{ij} = {\rm const}.
\end{align}
This implies that, even though the direction of $\mathbf{z}_{ij}(t)$ may vary in time, its norm $\|\mathbf{z}_{ij}\|$ is constant. 
Taking the time derivative of both sides of~\eqref{eq:rigid_body_dist} yields
$\dot{\mathbf{z}}_{ij}^\top\mathbf{z}_{ij} = 0$,
which implies that the directions of $\mathbf{z}_{ij}$ and $\dot{\mathbf{z}}_{ij}^\perp = Q\dot{\mathbf{z}}_{ij}$ coincide.
We can then explicitly decompose $\mathbf{z}_{ij}$ in two factors
\begin{align}\label{eq:z_ij_decomp}
 \mathbf{z}_{ij} 
 = d_{ij} \vec{\mathbf{y}}_{ij},
\end{align}
where $\vec{\mathbf{y}}_{ij}= \dot{\mathbf{z}}_{ij}^\perp\big/ \|\dot{\mathbf{z}}^\perp_{ij}\|\in\mathbb{S}^1$ is the unit vector denoting the time-varying oriented line (axis) along which $\mathbf{z}_{ij}$ lies, and $d_{ij}\in\mathbb{R}$ is the coordinate of $\mathbf{z}_{ij}$ on $\vec{\mathbf{y}}_{ij}$.

Let each agent $i$ send to all $j \in \mathcal{N}_i$ the (measured) velocity of its contact point $\mathbf{v}_{C_i}$ using the one-hop communication links. Then, each agent~$i$ can compute the velocity difference 
\begin{align}
\dot{\mathbf{z}}_{ij} = \mathbf{v}_{C_i}-\mathbf{v}_{C_j},
\label{eq:local_vel_difference}
\end{align}
and the corresponding orthogonal vector $\dot{\mathbf{z}}_{ij}^\perp $, for each $j\in{\cal N}_i$. As a consequence, $\vec{\mathbf{y}}_{ij}$ is locally available to each agent~$i$, $\forall j\in{\cal N}_i$. This is the first milestone of our algorithm, which is formally stated in the following result.
\begin{result}\label{res:y_ij}
The axis $\vec{\mathbf{y}}_{ij}$ along which $\mathbf{z}_{ij}$ lies is directly computed from local measurements and one-hop communication as
$\vec{\mathbf{y}}_{ij}
=
{\dot{\mathbf{z}}}_{ij}^\perp / \|\dot{\mathbf{z}}_{ij}^\perp\|
=
(\mathbf{v}_{C_i} - \mathbf{v}_{C_j})^\perp / \| \mathbf{v}_{C_i} - \mathbf{v}_{C_j}\| $,
as long as $\|\dot{\mathbf{z}}_{ij}\|=\|\mathbf{v}_{C_i} - \mathbf{v}_{C_j}\|\neq 0$.
\end{result}

\medskip
To obtain the sought $\mathbf{z}_{ij}$, 
only the estimation of $d_{ij}$ is left. 
Due to the rigid body constraint~\eqref{eq:rigid_body_dist}, 
$|d_{ij}| = \|\mathbf{z}_{ij}\| = {\rm const}$
holds, i.e., $d_{ij}$ is either equal to $\|\mathbf{z}_{ij}\|$ or to $-\|\mathbf{z}_{ij}\|$, depending on the fact that $\vec{\mathbf{y}}_{ij}$ and $\mathbf{z}_{ij}$ have the same or the opposite direction. 
However, since in~\eqref{eq:z_ij_decomp} both $\mathbf{z}_{ij}(t)$ and $\vec{\mathbf{y}}_{ij}(t)$ are continuous functions of time (for $\vec{\mathbf{y}}_{ij}$ this holds in any open interval in which $\|\dot{\mathbf{z}}_{ij}\|\neq 0$), we have that $\mbox{sign}(d_{ij}) = {\rm const} \;\; \forall t\in T$, as long as $\|\dot{\mathbf{z}}_{ij}\| \neq 0, \;\; \forall t\in T$.
Thus, in any time interval $T$ in which $\|\dot{\mathbf{z}}_{ij}\|\neq 0$ and under the reasonable assumption that the input wrenches are continuous in $t$ over $T$, we can differentiate both sides of~\eqref{eq:z_ij_decomp}, thus obtaining
\begin{align}\label{eq:diff_eq:z_ij_decomp}
\dot{\mathbf{z}}_{ij} = d_{ij} \dot{\vec{\mathbf{y}}}_{ij},
\end{align}
which forms a linear estimation problem that agent~$i$ can locally solve to estimate the sought $d_{ij}$. In fact,
among the quantities that appear in~\eqref{eq:diff_eq:z_ij_decomp}, 
agent~$i$ knows the quantity $\dot{\mathbf{z}}_{ij}$ and the time integral of $\frac{\rm d}{\rm d t}\vec{\mathbf{y}}_{ij}$, i.e., $\vec{\mathbf{y}}_{ij}$. Therefore, the estimate of $d_{ij}$ can be carried out using a standard online linear estimation technique described, e.g., in~\cite{1991-SloLi_}, and summarized in the Appendix of~\cite{2014k-FraPetRiz}. 
This technique has also the property of averaging out the possible measurement noise. To this aim, the time interval $T$ can be tuned on the basis of the noise level that has to be averaged out in the velocity measurements. 

Note that after the first estimation of $d_{ij}$ there is no need to further estimate $|d_{ij}|$, since this is a constant quantity. Thus, the only signal to keep track of is ${\rm sign}(d_{ij})$. This can be instantaneously achieved by implementing two linear observers of the dynamic system~\eqref{eq:diff_eq:z_ij_decomp}: one that assumes ${\rm sign}(d_{ij})=1$ and the other assuming ${\rm sign}(d_{ij})=-1$. Then, it is sufficient to select, at each time-step, the observer that provides the best estimate in terms of, e.g., measurement residual. 

To conclude the description of the algorithm, every time it happens to be $\|\dot{\mathbf{z}}_{ij}\|=0$, the last estimate of $\mathbf{z}_{ij}$ is kept frozen. In a real implementation the introduction of a suitable threshold to cope with the possible noise is recommended.

We summarize the results of this section in the following.
\begin{result}\label{res:z_ij}
The vector $\mathbf{z}_{ij}$ is estimated locally by agent $i$ and $j$ by the separate computation of two quantities
\begin{itemize}
\item $\vec{\mathbf{y}}_{ij}$ (time-varying axis), computed directly from $\mathbf{v}_{C_i} - \mathbf{v}_{C_j}$ (see Result~\ref{res:y_ij})
\item $d_{ij}$ (norm-constant coordinate along $\vec{\mathbf{y}}_{ij}$), computed from $\mathbf{v}_{C_i} - \mathbf{v}_{C_j}$ by  solving~\eqref{eq:diff_eq:z_ij_decomp} via filtering and applying online Linear Least Squares (LLS).
\end{itemize} 
\end{result} 

This part of the algorithm is referred with blocks $1,2,3$, and $4$ in the diagram of Fig.~\ref{fig:AlgoOverview}.

\subsection{Estimation of $\mathbf{z}_{i}(t)$ and $\omega(t)$}\label{sec:est_z_i_omega}

The estimated quantities $\mathbf{z}_{ij}(t)$ provide a straightforward way to estimate $\mathbf{z}_i$. This estimation phase corresponds to block $5$ in the diagram of Fig.~\ref{fig:AlgoOverview}. 
\begin{result}\label{res:z_i}
Once the estimate of $\mathbf{z}_{ij}(t)$ is available to each agent~$i$, $\forall j\in{\cal N}_i$, each agent~$i$ estimates $\mathbf{z}_{i}$ by using the centroid estimation algorithm described in~\cite{2012-AraCarSagCal}.
\end{result}

In order to estimate the angular rate $\omega$, we use the following relation from rigid body kinematics
\begin{align}
\omega \mathbf{z}_{ij} = -\dot{\mathbf{z}}_{ij}^\perp, \label{eq:rigid_body_omega}
\end{align}
where
 $\dot{\mathbf{z}}_{ij}^\perp$ is locally computed from~\eqref{eq:local_vel_difference}, and $\mathbf{z}_{ij}$ is locally estimated, as shown in Sec.~\ref{sec:est_z_ij}.
Multiplying both sides of~\eqref{eq:rigid_body_omega} by $\mathbf{z}_{ij}^\top$, we obtain that, for each pair of communicating agents $i$ and $j$, an estimate of $\omega$ is directly given by
\begin{align}
\omega = -\left(\mathbf{z}_{ij}^\top\dot{\mathbf{z}}_{ij}^\perp\right)\left(\mathbf{z}_{ij}^\top\mathbf{z}_{ij}\right)^{-1}. \label{eq:omega_estimation}
\end{align}
\begin{result}\label{res:w_estim}
$\omega$ is locally computed using~\eqref{eq:omega_estimation}, where $\dot{\mathbf{z}}_{ij}$ comes from direct measurement and one-hop communication and $\mathbf{z}_{ij}$ from Result~\ref{res:z_ij}.
\end{result}
This part of the algorithm corresponds to block $6$ in the diagram of Fig.~\ref{fig:AlgoOverview}.
The use of~\eqref{eq:omega_estimation} provides agent~$i$ with as many estimates of $\omega$ as the number of its one-hop neighbors $|{\cal N}_i|$. In the ideal case of noiseless velocity measurements, all those estimates are identical. In the case of noisy velocity measurements, this redundancy can be exploited to average out the noise either at the local level (e.g., by averaging the different estimates corresponding to each neighbor) or at the global level (by, e.g., using some dynamic consensus strategy among all agents~\cite{2010-ZhuMar}). \rev{For the sake of presentation clarity we restrain ourselves from presenting these minor details here.} %
Clearly, the order of the dynamic consensus algorithm used is strictly related to the time variations of $\omega$ and, therefore, to the time variations of its estimates~\cite{2010-ZhuMar}. Moreover, such consensus will theoretically converge asymptotically. %
However, dynamic consensus algorithms, able to track the average of their dynamic inputs up to a given bound, can be used~\cite{2013-KiaCorMar}.
Clearly, a measurement of the angular rate $\omega$ can also be obtained equipping each agent with a gyroscope placed at the contact point. Nevertheless, one of the goals, and contributions, of our work is to show that this additional sensor is not strictly needed  to accomplish the estimation task.

\begin{remark}
This estimation approach relies on the agreement on a common reference frame. In fact, the measured velocities $\mathbf{v}_{C_{i}}$ used to estimate $\mathbf{z}_i$, are referred to the same reference frame. Two possible approaches can be put forward in real-world applications: \emph{i)} agents should communicate to agree on a common reference frame, or \emph{ii)} agents use additional sensors (i.e., vision,
compass, or infrared array) to perform conversions between quantities related to different frames. \rev{[AR:PS MA siamo sicuri che le norme non si possono sommare? INOLTRE: abbiamo reference a supporto dei precedenti punti i) e ii) ? ]
\blue{AP: lo stesso vettore espresso in due frame di riferimento diversi ha due diverse norme, no? Ho in mente il caso di un punto mobile. A seconda che il punto sia descritto da un vettore in un frame mobile o in uno fisso, cambia il modulo del vettore associato.}
\red{AF: sbagliato, infatti la norma non cambia da un sistema di riferimento all'altro, quindi il consensus si può fare anche con sistemi di rif diversi. Tuttavia l'applicazione della forza $\mathbf{f}_i=k_z\mathbf{z}_i^\perp$ si basa sul fatto che ognuno ha stimato $\mathbf{z}_i$ con il metodo descritto prima che suppone il sistema di riferimento comune.}} 
\label{remark:Ref_Frame}
\end{remark}

\section{Dynamical Phase}\label{sec:dyn_phase}

The objective of this phase (corresponding to the dashed red box in Fig.~\ref{fig:AlgoOverview}) is to estimate the remaining quantities, i.e., the (constant) rotational inertia $J$, the (time-varying) position $\mathbf{z}_C(t)$ of the CoM of $B$ relative to the geometric center $G$ of the contact points, the (time-varying) velocity of the CoM $\mathbf{v}_C(t)$, and the (constant) mass $m$. The order in which they are estimated follows a dependency hierarchy, since some phase needs information about the outcome of previous ones. Thus, the order of estimation cannot be altered without preventing the correct execution of the proposed strategy. 
This phase makes use of the velocity measurements $\mathbf{v}_{C_i}$, the applied wrench $\mathbf{u}_{i}$, as well as the rigid body kinematics and dynamics.
The basic operations executed in this phase are summarized in the following: 
\begin{enumerate}
\item \emph{(estimation of $J$)} we exploit the knowledge of $\mathbf{z}_i$ to apply a particular input wrench that cancels the effect of $\mathbf{z}_C$ in~\eqref{eq:dynamics_rewritten}, thus, obtaining a reduced dynamics in which $J$ is the only unknown; then, we estimate $J$ using LLS;
\item \emph{(estimation of $\mathbf{z}_C(t)$)} we use all the previously estimated quantities, the rotational dynamics in~\eqref{eq:dynamics_rewritten}, and the rigid body constraint to recast the estimation of $\mathbf{z}_C$ to a nonlinear observation problem that can be locally solved by each agent with an observer designed by us;
\item \emph{(estimation of $\mathbf{v}_C(t)$)} we use rigid body kinematics to compute $\mathbf{v}_C(t)$ from all the quantities estimated so far;
\item \emph{(estimation of $m$)} we use a distributed estimation of the total force produced by the agents and $\mathbf{v}_C(t)$ to finally estimate the constant $m$ using LLS.
\end{enumerate}

\subsection{Estimation of $J$}\label{sec:est_J}

Assuming $J$ as a constant quantity, our strategy is to impose a specific wrench for a short time interval in order to let its estimate converge close enough to the real value.
After this finite time interval, any wrench can be applied again. This feature enables the concurrent execution of estimation and ordinary manipulation tasks. 

Let us isolate the rotational dynamics from~\eqref{eq:dynamics_rewritten} 
\begin{align}
\dot{\omega} &= \frac{1}{J}\sum_{i=1}^n {\mathbf{z}_i^\perp}^\top \mathbf{f}_i + \frac{1}{J}{\mathbf{z}_C^\perp}^\top \sum_{i=1}^n \mathbf{f}_i + \frac{1}{J} \sum_{i=1}^n\tau_i,
\label{eq:rotational_dynamics} 
\end{align}
where:
\begin{inparaenum}[\em (i)]
\item $J$ is the constant to be estimated;
\item $\omega(t)$ is locally known by each agent thanks to Result~\ref{res:w_estim};
\item $\mathbf{z}_i(t)$ is locally known by each agent thanks to Result~\ref{res:z_i};
\item $\mathbf{f}_i$ and $\tau_i$ are locally known by each agent, since they are applied by the agent itself;
\item $\mathbf{z}_C$ is still unknown. 
\end{inparaenum}
If we were able to eliminate $\mathbf{z}_C$ from~\eqref{eq:rotational_dynamics},
then $J$ would become the only unknown in~\eqref{eq:rotational_dynamics}. %
It is easy to verify that the influence of $\mathbf{z}_C$ in~\eqref{eq:rotational_dynamics} is eliminated if each agent~$i$ applies a force $\mathbf{f}_i$ such that $\sum_{i=1}^n \mathbf{f}_i=0$. 
A possible choice is to set $\mathbf{f}_i=k_z\mathbf{z}_i^\perp$, where $k_z\neq 0$ is an arbitrary constant. In fact, this choice implies 
$\sum_{i=1}^n \mathbf{f}_i = k_z\sum_{i=1}^n \mathbf{z}_i^\perp = k_z\sum_{i=1}^n (\mathbf{p}_{C_i}-{\mathbf{p}_{G}})^\perp = k_z Q \sum_{i=1}^n (\mathbf{p}_{C_i}-{\mathbf{p}_{G}}) = 0$.
Note that this force can be computed by each agent in a distributed way, since $\mathbf{z}_i$ is locally known thanks to Result~\ref{res:z_i}.

By applying $\mathbf{f}_i=k_z\mathbf{z}_i^\perp$, the rotational dynamics~\eqref{eq:rotational_dynamics} becomes
$\dot{\omega} = 
k_z J^{-1}\sum_{i=1}^n \|\mathbf{z}_i\|^2 + \frac{1}{J}\sum_{i=1}^n \tau_i$.
In order to further simplify the distributed computation, let us also impose $\tau_i=0$, $\forall i=1\ldots n$, limited to the time interval in which $J$ is estimated. Hence,~\eqref{eq:rotational_dynamics} is further simplified in
\begin{align}
\dot{\omega} = J^{-1} \; k_z\sum_{i=1}^n \|\mathbf{z}_i\|^2.\label{eq:J_lls_problem}
\end{align}
Equation~\eqref{eq:J_lls_problem} expresses a linear relation where the only unknown is the proportionality factor $J^{-1}$. In fact, $\omega$ and $k_z\Vert \mathbf{z}_i\Vert^2$ are locally known to each agent $i$, which implies that the constant quantity $k_z{\sum_{i=1}^n \Vert \mathbf{z}_i\Vert^2}$ can be computed distributively through an average consensus~\cite{2007-OlfFaxMur} right after the moment in which each agent is able to estimate $\mathbf{z}_i$ (block $7$ in the diagram of Fig.~\ref{fig:AlgoOverview}).
Therefore, the estimation of $J$ is recast in~\eqref{eq:J_lls_problem} as a LLS estimation problem that can be solved resorting to the same strategy used to estimate $d_{ij}$ in~\eqref{eq:diff_eq:z_ij_decomp}
(block $8$ in the diagram of Fig.~\ref{fig:AlgoOverview}). A summary follows.

\begin{result}
Each agent distributively computes the constant sum $k_z\sum_{i=1}^n \|\mathbf{z}_i\|^2$ using $\mathbf{z}_i$ from Result~\ref{res:z_i} followed by average consensus. Then, each agent~$i$ applies a force $\mathbf{f}_i=k_z\mathbf{z}_i^\perp$ for a given time interval, the moment of inertia $J$ is distributively computed by solving the LLS problem~\eqref{eq:J_lls_problem}. 
\end{result}

Ideally, in a noise-free setting, every agent concludes this phase with the same estimate of $J$. In realistic settings, where noise is present, each agent may have a slightly different estimate of $J$. Hence, a standard average consensus algorithm can be executed to average out the noise and improve the estimate of $J$. Also, in this case such consensus will theoretically converge asymptotically, %
but the convergence to a bounded ball centered in the average can be achieved in finite time and detected by means of suitable distributed strategies~\cite{2007-YadSal}.

\subsection{Estimation of $\mathbf{z}_C$}\label{sec:est_z_C(t)}
The main idea behind the estimation of the time-varying quantity $\mathbf{z}_C(t)$ is to rewrite~\eqref{eq:rotational_dynamics} in order to let only the following kinds of quantities appear (in addition to $\mathbf{z}_C(t)$):
\begin{itemize}
\item global quantities that can be distributively estimated;
\item local quantities available from the problem setting (measurements or inputs) or from the previous results.
\end{itemize}
Next, we demonstrate that such a rewriting is possible and also that the estimation of $\mathbf{z}_C(t)$ boils down to a solvable nonlinear observation problem.
Let us first decompose the local force $\mathbf{f}_i(t)$ in two parts and recall two important identities
\begin{align}
\mathbf{f}_i(t) = \frac{1}{n} &\sum_{i=1}^n \mathbf{f}_i(t) + \Delta \mathbf{f}_i(t) =\mathbf{f}_{\text{\rm mean}}(t) + \Delta \mathbf{f}_i(t),\label{eq:f_mean}\\
&\sum_{i=1}^n {\mathbf{z}_i^\perp}^\top=0 \quad\text{and}\quad
 \sum_{i=1}^n \Delta \mathbf{f}_i=0.\label{eq:zero_ident}
\end{align}
We can then rewrite~\eqref{eq:rotational_dynamics}, using~\eqref{eq:f_mean} and~\eqref{eq:zero_ident}, as 
$ \dot \omega =  J^{-1} \left(\sum_{i=1}^n {\mathbf{z}_i^\perp}^\top\right) \mathbf{f}_{\text{\rm mean}} + n J^{-1}{\mathbf{z}_C^\perp}^\top \mathbf{f}_{\text{\rm mean}}(t)+ J^{-1} \sum_{i=1}^n {\mathbf{z}_i^\perp}^\top\Delta \mathbf{f}_i +
+ \, J^{-1}{\mathbf{z}_C^\perp}^\top \sum_{i=1}^n\Delta \mathbf{f}_i 
+ J^{-1}\sum_{i=1}^n \tau_i =
 \underbrace{n J^{-1}{\mathbf{z}_C^\perp}^\top \mathbf{f}_{\text{\rm mean}}}_{{\mathbf{z}_C^\perp}^\top \bar{\mathbf{f}}}
+
\underbrace{J^{-1} \sum_{i=1}^n {\mathbf{z}_i^\perp}^\top \Delta \mathbf{f}_i}_{\eta_1} + \underbrace{J^{-1}\sum_{i=1}^n \tau_i}_{\eta_2}$, 
i.e.,
\begin{align}
\dot \omega = \,
{\mathbf{z}_C^\perp}^\top \bar{\mathbf{f}}
+ \eta_1 + \eta_2. 
\label{eq:omegadot}
\end{align}
The global quantities\footnote{The same considerations made in Remark~\ref{remark:Ref_Frame} are in order.} 
\begin{itemize}
\item $\bar{\mathbf{f}}= \dfrac{n}{J}\mathbf{f}_{\rm mean}$, 
\item $\eta_1=J^{-1} \sum_{i=1}^n {\mathbf{z}_i^\perp}^\top \Delta \mathbf{f}_i = J^{-1} \sum_{i=1}^n {\mathbf{z}_i^\perp}^\top \mathbf{f}_i$, and
\item $\eta_2=J^{-1}\sum_{i=1}^n \tau_i$
\end{itemize}
can be all distributively estimated in parallel using three instances of the dynamic consensus algorithm~\cite{2010-ZhuMar} (blocks $9$, $10$, and $11$ in the diagram of Fig.~\ref{fig:AlgoOverview}). %
The choice of the specific dynamic consensus algorithm strongly depends on the nature of the tracked signals. 
A dynamic consensus algorithm like the one presented in~\cite{2013-KiaCorMar} enables the estimate of its convergence time given the rate of convergence. The only mild assumption made is that the input signals are continuous and bounded~(in~\cite{2013-KiaCorMar}, Theorem 5.1).
The global quantity $\omega(t)$ is known thanks to Result~\ref{res:w_estim}.
The only unknown in~\eqref{eq:omegadot} is $\mathbf{z}_C(t)$.
Define $\eta = \eta_1 + \eta_2$. The following result holds:
\begin{result}\label{res:recast_omega_dyn}
The rotational dynamics is given by 
\begin{align}
\dot \omega = {\mathbf{z}_C^\perp}^\top \bar{\mathbf{f}} + \eta, \label{eq:omega_dyn_z_C}
\end{align}
where $\omega(t)$ is known thanks to Result~\ref{res:w_estim} and $\bar{\mathbf{f}}(t)$ and $\mathbf{\eta}(t)$ are locally known to each agent through distributed computation.
\end{result}

We use Eq.~\eqref{eq:omega_dyn_z_C} to form a dynamical system where $\omega$ and $\mathbf{z}_C(t)$ are the state variables and $\bar{\mathbf{f}}$ and $\eta$ are the inputs.
Recalling 
that 
$\mathbf{z}_C$ is a constant-norm vector, rigidly attached to the object, the following holds:
\begin{align}
\dot{\mathbf{z}}_C= \omega\,\mathbf{z}_C^\perp.\label{eq:zC_dyn}
\end{align}
Combining~\eqref{eq:omega_dyn_z_C} and~\eqref{eq:zC_dyn}, we obtain the nonlinear system
\begin{align}
\left\{
\begin{aligned}
\dot x_{1} &= -x_{2}x_{3} \\
\dot x_{2} &= x_{1}x_{3}\\
\dot x_{3} & = x_{1}u_2-x_{2}u_1 + u_3
\end{aligned}
\right., \quad y = x_3,
\label{eq:nonlinear_sys_zC}
\end{align}
where $\mathbf{z}_{C}=(z_{C}^x\;z_{C}^y)^\top=(x_{1}\;x_{2})^\top$ is the unknown part of the state vector, $\omega = x_{3}$ is the measured part of the state vector and, consequently, can be considered as the system output, and $\bar{\textbf{f}}=(\bar{f}_x\;\bar{f}_y)^\top = (u_1 \; u_2)^\top$, 
 $\eta = u_3$ are known inputs.

\begin{result}\label{res:reduced_prob_z_C}
Estimating $\mathbf{z}_C(t)$ is equivalent to observe the state $(x_{1}\;x_{2})^\top$ of the nonlinear system~\eqref{eq:nonlinear_sys_zC} with known output~$y=x_3=\omega$ and known inputs $u_1=\bar{f}_x$, $u_2=\bar{f}_y$, and $u_3=\eta$.
\end{result}

In~\cite{2015b-FraPetRiz}, we studied the observability of~\eqref{eq:nonlinear_sys_zC}:
\begin{prop} 
If $x_3 \not\equiv 0 $ and $\left( u_1 \; u_2\right)^\top \not\equiv \mathbf{0}$, then system~\eqref{eq:nonlinear_sys_zC} is locally observable in the sense of~\cite{1977-HerKre}. 
\end{prop}
\begin{proof} Given in~\cite{2015b-FraPetRiz}. 
\end{proof}

Note that the applied torques $\tau_i$, for $i=1 \ldots n$ (which are included in $u_3$) have no influence on the observability of $\mathbf{z}_C(t)$.
In~\cite{2013-Agh}, an observability condition that involves the angular velocity is also given. However, the condition expressed in~\cite{2013-Agh} pertains the estimation of the kinematic parameters and requires that the direction of the angular velocity does not remain constant over the time, while in our setting the direction is constant and we only require that the norm is not constantly zero.
In~\cite{2015b-FraPetRiz}, we also proposed a nonlinear observer for system~\eqref{eq:nonlinear_sys_zC}, which is summarized in the following result. 

\begin{prop}\label{prop:observer}
Consider the following dynamical system
\begin{equation}
\label{eq:cm_observer_eq}
\left\{
\begin{aligned}
\dot {\hat x}_1 &= -\hat{x}_2 x_3 + u_2(y - \hat{x}_3)\\
\dot {\hat x}_2 &= \hat x_1 x_3 - u_1(y - \hat{x}_3)\\
\dot {\hat x}_3 &= \hat x_1 u_2 - \hat x_2 u_1 + k_e(y-{\hat x}_3) + u_3
\end{aligned}
\right. ,
\end{equation}
where $k_e>0$.
If $y(t)\not\equiv 0$ and $\left( u_1(t) \; u_2\right(t))^\top \not\equiv \mathbf{0}$, then~\eqref{eq:cm_observer_eq} is an asymptotic observer for~\eqref{eq:nonlinear_sys_zC}. Hence, defining $\hat{\mathbf{x}}= ({\hat x}_1 \ {\hat x}_2 \ {\hat x}_3)^\top$ and $\mathbf{x}= ( x_1 \ x_2 \ x_3)^\top$, one has that $\hat{\mathbf{x}}(t) \rightarrow \mathbf{x}(t)$ asymptotically.
\end{prop}
\begin{proof}
Given in~\cite{2015b-FraPetRiz}.
\end{proof}

Thanks to Proposition~\ref{prop:observer} we can state the following result:
\begin{result}\label{res:z_C_estim}
The relative position of the CoM w.r.t. the center of the contact points, i.e., $\mathbf{z}_C(t)$, is distributively computed by using the observer~\eqref{eq:cm_observer_eq} and thanks to the local knowledge of $n$, $J$, $\omega$, $\mathbf{f}_{\text{\rm mean}}$, and $\sum_{i=1}^n {\mathbf{z}_i^\perp}^\top\Delta \mathbf{f}_i$ from the previous results.
\end{result}

The estimation of $\mathbf{z}_C(t)$ described beforehand is schematized in blocks $9,10,11$ (dynamic consensus algorithms) and $12$ (observer introduced in~\eqref{eq:cm_observer_eq}) in the diagram of Fig.~\ref{fig:AlgoOverview}. 

The inputs $u_{1,2,3}$ of the observer arise from a dynamic consensus phase. As \rev{deeply explained in Sec.~\ref{sec:cons_conv}} already stated, dynamic consensus algorithms converge asymptotically, thus, only the convergence to a ball centered in the average of the values of nodes can be guaranteed in finite time, in dependence of the convergence rate~\cite{2013-KiaCorMar}.
For this reason, it is important to analyze what happens to the observer's state when an additive input disturbance is present. To this aim, we define $\widetilde{u}_i = u_i + \varepsilon_i $, with $i=1,2,3$ and analyze the dynamics of $\mathbf{e} = \mathbf{x}-\hat{\mathbf{x}}$
\begin{equation}
\label{eq:error_observer_eq}
\left\{
\begin{aligned}
\dot {e}_1 &= -e_2 x_3 + \widetilde{u}_2 e_3\\
\dot {e}_2 &= e_1 x_3 - \widetilde{u}_1 e_3\\
\dot {e}_3 &= e_1 \widetilde{u}_2 - e_2 \widetilde{u}_1 - k_e e_3 - \varepsilon_3
\end{aligned}
\right. .
\end{equation}
{Defining the class $\mathcal{K}$ function\footnote{{According to Definition 4.2,~\cite{2002-Kha}, a continuous function $\alpha :[0,a)\rightarrow [0,\infty )$ is said to belong to class $\mathcal {K}$ if it is strictly increasing and $\alpha(0)=0$.}} $\gamma(r) = \frac{r}{k_e (\theta - 1)}$, where $0<\theta<1$ and $k_e>0$, the following result holds.}
\begin{prop}\label{prop:iss}
The error system of~\eqref{eq:cm_observer_eq}, i.e.,~\eqref{eq:error_observer_eq}, is Input-to-State Stable (ISS) with $\gamma(r) = \frac{r}{k_e (\theta - 1)}$, where $0<\theta<1$ and $k_e>0$, according to {Definition 4.7,~\cite{2002-Kha}}. 
\end{prop}
\begin{proof}
{Provided in the technical report associated with this paper, which can be downloaded at: \url{https://arxiv.org/abs/1602.01891}}
\end{proof}

\subsection{Estimation of $\mathbf{v}_C$}\label{sec:est_v_C(t)}

The velocity of the center of mass $\mathbf{v}_C(t)$ is estimated locally by each agent~$i$ using the rigid body constraint
$\frac{\rm d}{\rm dt}(\mathbf{p}_{C} - \mathbf{p}_{C_i}) = \omega(\mathbf{p}_{C} - \mathbf{p}_{C_i})^{\perp}$,
which can be rewritten as
\begin{align}
\mathbf{v}_{C}(t) &= \mathbf{v}_{C_i}(t) - \omega(t)(\mathbf{z}_C(t) + \mathbf{z}_i(t))^{\perp},
\label{eq:velocity_com}
\end{align}
whose right-hand-side elements are all known since:
\begin{itemize}
\item $\mathbf{v}_{C_i}(t)$ is locally measured by agent $i$
\item $\omega(t)$, $\mathbf{z}_C(t)$, and $\mathbf{z}_i(t)$ are known by each agent $i$ thanks to Results~\ref{res:w_estim},~\ref{res:z_C_estim}, and~\ref{res:z_i}, respectively.
\end{itemize}

\begin{result}\label{res:v_C_estim}
The CoM velocity $\mathbf{v}_C(t)$ is distributively computed using~\eqref{eq:velocity_com} and the knowledge of $\mathbf{v}_{C_i}(t)$, $\omega(t)$, $\mathbf{z}_C(t)$, $\mathbf{z}_i(t)$.
\end{result}
Block $13$ in the diagram of Fig.~\ref{fig:AlgoOverview} represents this part. 

\subsection{Estimation of the mass $m$}\label{sec:est_m}

 The estimation of the mass $m$ is a straightforward consequence of the estimation of the $\mathbf{v}_{C_i}(t)$ and average force. In fact, rewriting~\eqref{eq:dynamics_1} as %
$\dot{\mathbf{v}}_C = \frac{1}{m}\sum_{i=1}^n \mathbf{f}_i = \frac{n}{m}\mathbf{f}_{\rm mean}$, 
we obtain
\begin{align}
\dot{\mathbf{v}}_C = m^{-1}\; n\,\mathbf{f}_{\rm mean},
\label{eq:lls_mass}
\end{align}
where
\begin{inparaenum}[\it i)]
\item $n$ is known,
\item $\mathbf{f}_{\text{\rm mean}}$ is distributively estimated from $\mathbf{f}_i$ using dynamic consensus (as in Result~\ref{res:recast_omega_dyn}), and
\item $\mathbf{v}_C$ is known locally by each agent $i$ thanks to Result~\ref{res:v_C_estim}.
\end{inparaenum}
Thus, the problem is recast as the linear least square estimation problem~\eqref{eq:lls_mass} that can be solved resorting to the same strategy used to estimate $d_{ij}$ in~\eqref{eq:diff_eq:z_ij_decomp} and $J$ in~\eqref{eq:J_lls_problem}.

\begin{result}\label{res:m_estim}
The mass $m$ is distributively computed from the knowledge of $\mathbf{v}$ and $n$, and $\mathbf{f}_{\rm mean}$ by solving an online linear least square problem via filtering~\eqref{eq:lls_mass}.
\end{result}

Block $14$ in the diagram of Fig.~\ref{fig:AlgoOverview} represents this part.

\section{Inertia and Mass Changing during the Task}\label{sec:TimeVarInertParams}
In some particular cases, it might happen that the values of $J$ and $m$ change during the manipulation because, e.g., an object which was part of the load is dropped or, viceversa, a new object is added to the load. Such changes cause discrete jumps at certain instants in the values of $J$ and $m$, which could be estimated again using the methods proposed in Sections~\ref{sec:est_J} and~\ref{sec:est_m}. However, in order to do so the robots would need to detect that $J$ and $m$ have changed and to trigger again the estimation algorithms. Furthermore, the estimation of $J$ proposed in~\ref{sec:est_J} requires that each agent applies a specific force $\mathbf{f}_i=k_z\mathbf{z}_i^\perp$. However, in some cases it might be inconvenient to temporarily pause the manipulation and apply such forces for estimating again $J$.

In order to deal with such possibilities we propose here
a variant of the first estimation algorithm which is based on two new estimators, see Fig.~\ref{fig:AlgoOverviewNew}. The first is an alternative observer that extends~\eqref{eq:cm_observer_eq} including $J$ in the state vector. Such estimation method can run in the background during the manipulation, thus overcoming all the aforementioned possible pitfalls.
The second is a new observer for $m$ that can also run in the background and does not need any trigger or special coordination among agents. Both observers are derived next.

\begin{figure}[t]
\centering
\includegraphics[width=.95\columnwidth]{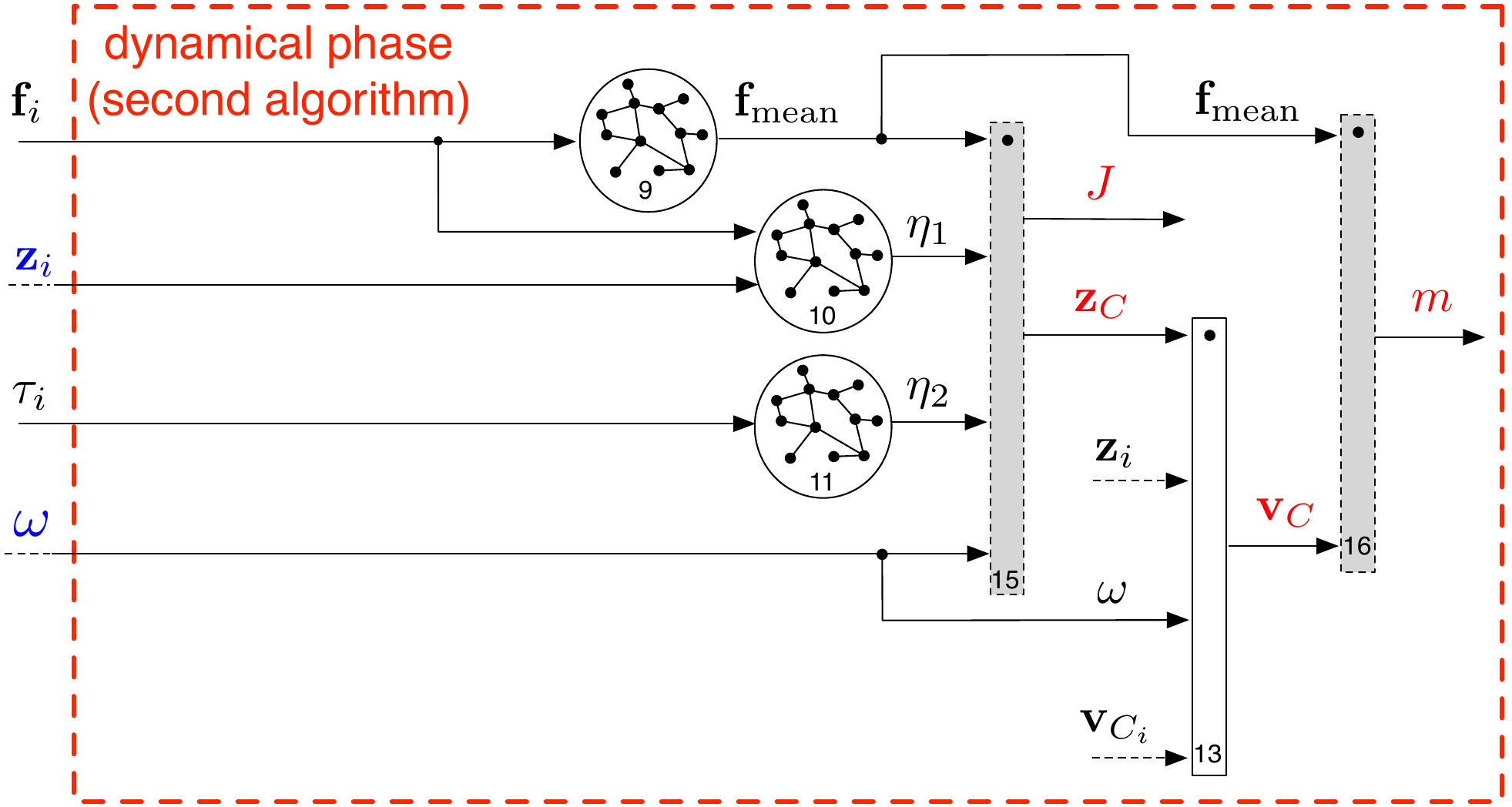}
\caption{Overview of the dynamical phase for the second algorithm. Compared to the dynamical part of the first algorithm (Fig.~\ref{fig:AlgoOverview}) the consensus $7$ and the step estimator $8$ blocks are removed; while block $12$ is replaced by the new estimator~\eqref{eq:cm_observer_eq_tvar} (new gray block $15$); and  block $14$ is replaced by the new estimator~\eqref{eq:m_observer_eq} (new gray block $16$). 
}
\label{fig:AlgoOverviewNew}
\end{figure}

\smallskip
Assuming that $J$ is not known, Result~\ref{res:recast_omega_dyn} can be recast as:
\begin{result}\label{res:recast_omega_dyn_J}
The rotational dynamics is given by 
\begin{align}
\dot \omega = J^{-1} {\mathbf{z}_C^\perp}^\top \tilde{\mathbf{f}} + J^{-1}\tilde{\eta}, \label{eq:omega_dyn_J_z_C}
\end{align}
where $\omega(t)$ is known thanks to Result~\ref{res:w_estim}, and $\tilde{\mathbf{f}}=n\mathbf{f}_{\rm mean}$ and $\tilde\eta=J\eta_1+J\eta_2 = \sum_{i=1}^n {\mathbf{z}_i^\perp}^\top \mathbf{f}_i + \sum_{i=1}^n \tau_i$ are locally known to each agent through distributed computation.
\end{result}

Combining~\eqref{eq:omega_dyn_J_z_C} and~\eqref{eq:zC_dyn}, and knowing the fact that $\dot{J}=0$ at any time except in the few isolated instants in which the unknown load inertia undergoes discrete, step-like  changes, we obtain an extended version of~\eqref{eq:nonlinear_sys_zC}
\begin{align}
\left\{
\begin{aligned}
\dot x_{1} &= -x_{2}x_{3} \\
\dot x_{2} &= x_{1}x_{3}\\
\dot x_{3} & = x_{4}x_{1}u_2-x_{4}x_{2}u_1 + x_{4}u_3\\
\dot x_{4} &= 0
\end{aligned}
\right., \quad\;\; y = x_3,
\label{eq:nonlinear_sys_J_zC}
\end{align}
where $(x_{1}\;x_{2})^\top=\mathbf{z}_{C}=(z_{C}^x\;z_{C}^y)^\top$ and $x_4=J^{-1}$ are the unknown parts of the state vector, $x_{3}=\omega$ is its measured part and, consequently, can be considered as the system output, and $(u_1 \; u_2)^\top=\tilde{\textbf{f}} $, 
 $u_3=\tilde\eta $ are known inputs.

 \begin{prop}\label{prop:observer_J}
Consider the following dynamical system
\begin{equation}
\label{eq:cm_observer_eq_tvar}
\begin{aligned}
\dot {\hat x}_1 &= -\hat{x}_2 x_3 + u_2(y - \hat{x}_3)\\
\dot {\hat x}_2 &= \hat x_1 x_3 - u_1(y - \hat{x}_3)\\
\dot {\hat x}_3 &= \hat x_1 u_2 - \hat x_2 u_1 + \hat x_4 u_3 + k_e(y-{\hat x}_3)\\
\dot {\hat x}_4 &= -u_3(y - \hat{x}_3),
\end{aligned}
\end{equation}
where $k_e>0$.
If $y(t)\not\equiv 0$, $u_3(t)\not\equiv 0$ and $\left( u_1(t) \; u_2(t)\right)^\top \not\equiv \mathbf{0}$, then~\eqref{eq:cm_observer_eq_tvar} is an asymptotic observer of a modified version of~\eqref{eq:nonlinear_sys_J_zC}, in which the following change of variable is performed $x_1\rightarrow x_1x_4$, $x_2\rightarrow x_2x_4$. Thus, defining $\hat{\mathbf{x}}= ({\hat x}_1 \ {\hat x}_2 \ {\hat x}_3 \ {\hat x}_4)^\top$ and $\mathbf{x}= ( x_1x_4 \ \ x_2x_4 \ \ x_3 \ \ x_4)^\top$, one has that $\hat{\mathbf{x}}(t) \rightarrow \mathbf{x}(t)$ asymptotically, which in turn implies that $({\hat x}_1/{\hat x}_4 \ \ {\hat x}_2/{\hat x}_4)\rightarrow (x_1 \ \ x_2)$ asymptotically.
\end{prop}

\begin{proof}
{Provided in the technical report associated with this paper, which can be downloaded at: \url{https://arxiv.org/abs/1602.01891}}
\end{proof}

In the ideal case, the estimates carried out by all the robots are identical.
In the case of noisy measurements, the noise can be averaged out by means of a dynamic consensus algorithm~\cite{2013-KiaCorMar}, as done for the estimation of $\omega$.

\smallskip
In order to observe $m$, we design an observer for the following nonlinear dynamical system
\begin{align}
\left\{
\begin{aligned}
\dot{z}_{1} &= z_{2}u \\
\dot z_{2} &= 0
\end{aligned}
\right., \quad \; y = z_1  ,
\label{eq:nonlinear_sys_m_bis}
\end{align}
which is easily derived from~\eqref{eq:lls_mass} by defining $z_1 = v_{C,x} + v_{C,y}$, $z_2 = m^{-1}$, $u = f_{{\rm mean},x} + f_{{\rm mean},y}$, and by imposing $\dot m = 0$ for the same reasons for which it was previously assumed $\dot J = 0$.

 \begin{prop}\label{prop:observer_m}
Consider the following dynamical system
\begin{equation}
\label{eq:m_observer_eq}
\begin{aligned}
\dot{\hat{z}}_1 &= \hat{z}_2 u + k_1(y - \hat{z}_1)\\
\dot {\hat z}_2 &= k_2(y - \hat{z}_1),
\end{aligned}
\end{equation}
where $k_1, k_2 \in \mathbb{R}$. If 
 $k_1 > 0$,
 $k_2u > 0$, and
$k_1^2 > 4k_2u$
hold, then~\eqref{eq:m_observer_eq} is an asymptotic observer of~\eqref{eq:nonlinear_sys_m_bis}.
\end{prop}
\begin{proof}
{Provided in the technical report associated with this paper, which can be downloaded at: \url{https://arxiv.org/abs/1602.01891}}
\end{proof}

The same observation made for the estimation of a time-varying $J$ is valid in this case: in the ideal case of noiseless measurements, the estimates carried out by all the robots are identical. In the case of noisy measurements, the noise can be averaged out by means of a dynamic consensus algorithm~\cite{2013-KiaCorMar}.

\section{Decentralized Observability-based Control}
\label{sec:SimpleLocalRules}

In the previous sections, we have shown how to solve the Problem in Sec.~\ref{sec:probStat}. Apart from the phase of the first-algorithm  in which $J$ is estimated (Sec.~\ref{sec:est_J}), in all the other phases we did not suggest any control input to move the load and perform the estimation. The user of the algorithm is free to use any control input, as long as it ensures the observability conditions, i.e., non-zero angular rate $\omega$ and non-zero average force $\mathbf{f}_{\rm mean}$. In each phase, either one or both of the two conditions are needed to ensure a convergent estimation. 

In the following, we prove that an extremely basic control strategy satisfies, under very mild conditions, the aforementioned observability requirements. Furthermore, this control strategy: %
\begin{inparaenum}[\it (i)]
\item can be implemented relying only on local perception and communication (it is, therefore, distributed); and 
\item does not require the knowledge of parameters and quantities that are the objectives of the distributed estimation (it is estimation-`agnostic'). Hence, it can be applied during the estimation process and independently from it.
\end{inparaenum}

\begin{prop} \label{prop:bounded_nonvanishing_omega}
Assume that the following local control rule is used: $\mathbf{f}_i=\mathbf{f}^*= {\rm const}$, $\tau_i=0$, $\forall i=1 \ldots n$, and denote with $\omega_0$ the rotational rate at $t=0$, then
\begin{enumerate}
\item $\omega(t)$ remains bounded, in particular:
\begin{align}
|\omega(t)| \le \sqrt{\omega_0^2 + 4nJ^{-1} \|\mathbf{f}^*\| \|\mathbf{z}_C\|} \quad \forall t\geq 0\label{eq:omega_bound}
\end{align}
\item $\exists {\bar t} \geq 0$ such that $\omega$ becomes $\omega\equiv0$ $\forall t\ge {\bar t}$, if and only if the following condition hold
\begin{align}
2 n \, \textbf{z}_C(0)^\top\textbf{f}^* - J\omega^2(0) = 2 n \, \|\textbf{z}_C\|\|\mathbf{f}^*\|.		 
		 \label{eq:critical_init_cond}
\end{align}
\end{enumerate}
Thus, the proposed control law is suitable for the estimation process apart from a zero-measure case implied by~\eqref{eq:critical_init_cond}.
\label{PropositionVII2}
\end{prop} 
\begin{proof}
{Provided in the technical report associated with this paper, which can be downloaded at: \url{https://arxiv.org/abs/1602.01891}}
\end{proof}

The use of the local control action in Proposition~\ref{prop:bounded_nonvanishing_omega} ensures the sought observability conditions under the very mild conditions specified in~\eqref{eq:critical_init_cond}. However, it causes the load CoM velocity to grow linearly over time (see, e.g.,~\eqref{eq:lls_mass}). Therefore, it is wise to modify the proposed control action by periodically changing the direction of the common force (i.e., switching between $\mathbf{f}^*$ and $-\mathbf{f}^*$ on a periodical basis). In this way, the CoM velocity will boundedly oscillate around zero.

It is also important to define a control strategy that is able to stop the load motion if needed (like, e.g., at the end of all the estimation phases). This is provided in the following. %

\begin{prop} \label{prop:stopping_motion}
Assume that the following local control rule is used: $\mathbf{f}_i=-b \mathbf{v}_{C_i}$, $\tau_i=0$, $\forall i=1 \ldots n$, with $b>0$. Then, both $\omega$ and $\mathbf{v}_C$ converge asymptotically to zero with a convergence rate that is proportional to $b$.

\end{prop}
\begin{proof}
{Provided in the technical report associated with this paper, which can be downloaded at: \url{https://arxiv.org/abs/1602.01891}}
\end{proof}

\section{Numerical Results}\label{sec:numRes}
This section and the Appendix show numerical results that validate our approach. 
First, we demonstrate the working principles of the first algorithm simulating a network of $n=10$ agents manipulating an unknown load on a plane. Then, the validity of the first algorithm is extensively assessed through a detailed simulation campaign in a wide range of operational conditions. Finally we show some simulations for the second algorithm using  time-varying mass and inertia observers.
\subsection{Manipulation of an unknown load by a team of 10 agents}
\label{sec:simulations_A}
We consider a planar load with $m=50$\,kg and $J=86.89$\,kg\,m$^{2}$, manipulated by a team of $n=10$ agents communicating over a line-topology network. Such a topology is the worst case for the algorithm convergence rate, which is an increasing function of the network diameter~\cite{2011-OlsTsi}. %
The velocity measurements are affected by an additive zero-mean Gaussian noise with covariance matrix $\boldsymbol\Sigma_{i} = \sigma^{2} \mathbf{I}_{2\times 2}$, and $\sigma = 0.3$ m/s. 
The quantities involved in the execution and assessment of the first algorithm are illustrated in Fig.~\ref{fig:estimation}.
As a first step, starting from $t_0=0$, each agent applies an arbitrary force and executes the procedure described in Sec.~\ref{sec:est_z_ij} to estimate the relative distances between contact points. We observe that the presence of noise in the velocity measurements can make the signal-to-noise ratio too small to make an effective use of acquired measurement. In this case, we opt for keeping the last valid measurement as constant, and not to update the measurement with quantities that are too noisy to be useful. 
The signal-to-noise ratio threshold in order to consider the  measured velocity is set to $\|\dot{\mathbf{z}}_{ij}\| \leq 0.5$ m/s. 
The first plot of Fig.~\ref{fig:estimation} reports the convergence to zero of the Estimation Error Relative Distance (EERD) index,  defined as EERD$(t)=\sum_{i=1}^{n-1} \sum_{j=i}^n \mathcal{A}(i,j) \|(\mathbf{z}_{ij}(t)-\hat{\mathbf{z}}_{ij}(t))\|_2$, where variable $\hat{\mathbf{z}}_{ij}$ indicates the estimate of $\mathbf{z}_{ij}$. Consistently, here and henceforth, the estimate of a quantity $\star$ is indicated with a superimposed hat, i.e., $\hat{\star}$. %
Dynamic consensus blocks have been implemented by means of the Fist Order Input Dynamic Consensus Filter~\cite{2013-KiaCorMar} with parameter values set as $\epsilon = 0.01$ and $\beta = 10$. Static average consensus detection is detected by means of the method presented in~\cite{2007-YadSal}.%
\begin{figure}[t]
\centering
\includegraphics[width=0.45\textwidth]{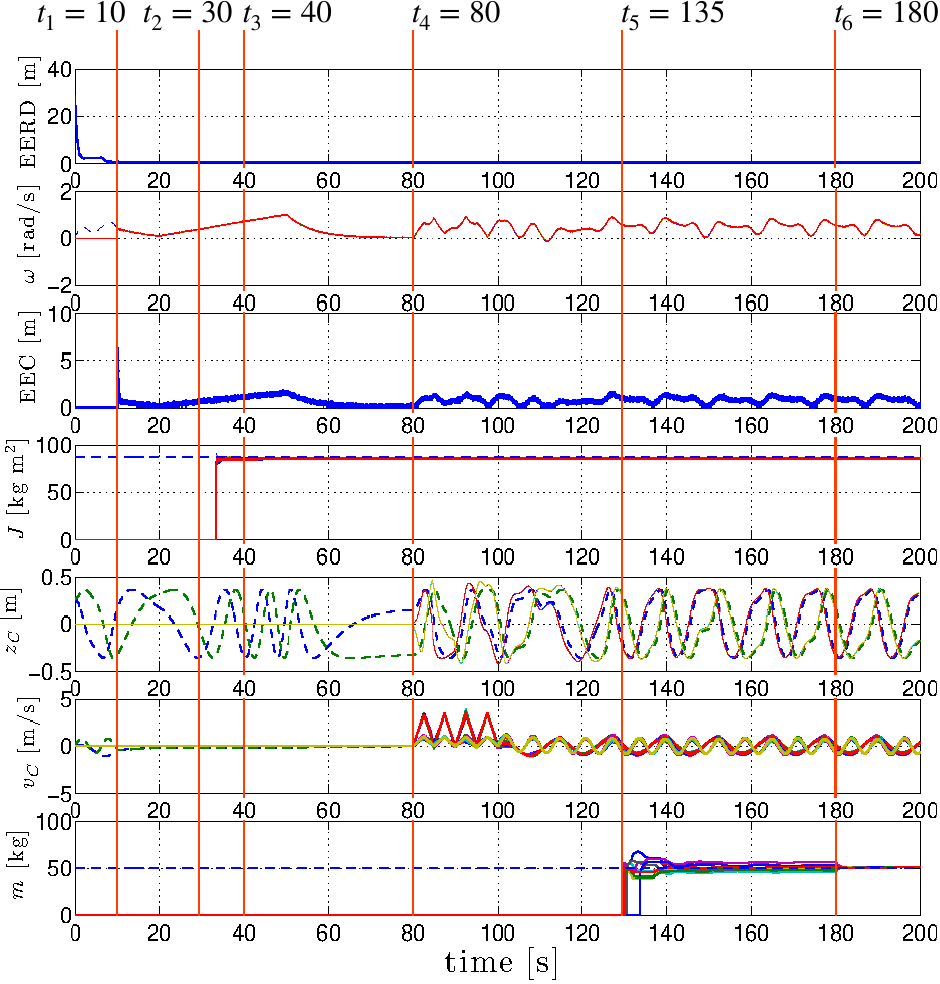}
\caption{An illustrative simulation of the whole  estimation algorithm described in Secs.~\ref{sec:kyn_phase} and~\ref{sec:dyn_phase}. From top to bottom, respectively: the trend of the EERD index, the estimates of $\omega$, the trend of the EEC index, the estimates of $J$, the observations of $\mathbf{z}_C$, the estimates of $\mathbf{v}_C$, and the estimates of $m$. Dashed lines indicate true values, while solid lines indicate estimates.}
\label{fig:estimation}
\end{figure}
Starting from  $t_1=10$\,s, each agent applies the control rules given in Propositions~\ref{prop:bounded_nonvanishing_omega}~and~\ref{prop:stopping_motion}, which guarantee both the observability and boundedness of $\left[\mathbf{v}_C^\top \,\, \omega\right]^\top$. At $t_1=10$\,s, %
  $\omega$ and $\mathbf{z}_i$ start to be estimated, as described in Sec.~\ref{sec:est_z_i_omega}. The second and third plots of Fig.~\ref{fig:estimation} illustrate, respectively, the trend of the angular velocity $\omega$ and its estimate $\hat{\omega}$, and the quadratic performance index on the estimation of the relative distance to the center of mass, i.e., EEC$(t)=\sum_{i=1}^{n} \|(\mathbf{z}_{i}(t)-\hat{\mathbf{z}}_{i}(t))\|_2$. 
At $t=20$\,s, the first step of the dynamical phase is executed, as described in Sec.~\ref{sec:est_J}. First, each agent runs an average consensus in order to locally estimate the constant value $k_z{\sum_{i=1}^n \Vert \mathbf{z}_i\Vert^2}$. Such consensus will theoretically converge
asymptotically. %
However, we use the technique presented in~\cite{2007-YadSal} to assess a suitable stopping condition. Specifically, we distributively determine when the consensus has been reached within a given error bound, which in our experiment is set to $0.001$\,m.%
 \rev{of practically converge that can be distributively
detected with several techniques~\cite{2016-KenKenSko}~\cite{2011-OlsTsi}.}
At $t_3=30$\,s each agent $i$ runs  a least square estimation of $J$ using also the knowledge of $\hat{\omega}$. 
Each agent checks the convergence of the least squares estimation evaluating the variance of the estimator~\cite{2005-Gee}. 
From $t_3=40$\,s, the local estimates are transmitted over the network and an average consensus is run to agree on a common estimate, which in our case is $\hat{J}=85.67$\,kg\,m$^2$ (fourth plot in Fig.~\ref{fig:estimation}). %
Also in this case, we use the technique presented in~\cite{2007-YadSal} with error bound set to $0.01$\,kg\,m$^2$. %
Then, the angular rate is brought to zero (Proposition~\ref{prop:stopping_motion}). 
Afterwards, at $t_4=80$\,s each agent starts
 the nonlinear observation of $\mathbf{z}_C$ described in Sec.~\ref{sec:est_z_C(t)}. The  observer errors reach zero at about $t_5=135$\,s, as illustrated in the fifth plot of Fig.~\ref{fig:estimation}. 
The estimate $\hat{\mathbf{v}}_C$ is then computed using~\eqref{eq:velocity_com} (sixth plot of Fig.~\ref{fig:estimation}), which in turn allows to compute $\hat{m}$, as explained in Sec.~\ref{sec:est_m}, by a preliminary collections of samples  and local least squares estimations. An average consensus phase, starting at $t_6=180$\,s, leads to an accurate estimate of $m$ at $t_f=200$\,s (seventh and last plot of Fig.~\ref{fig:estimation}). %
The same techniques previously described to detect consensus convergence~\cite{2007-YadSal} is applied here using a bound of $0.01$\,kg. %
The duration of the entire algorithm is $200$\,s, of which a large portion is needed  to collect samples to run the local least squares and the consensus algorithms for the constant parameters $m$, $J$ and $d_{ij}$. The duration of these phases depends on the noise level. Ideally, in the absence of noise, a single sample would be sufficient to perform the estimation, while in the real, noisy case, a trade-off between robustness~\cite{1991-SloLi_} and duration of the estimation phase is necessary.
Finally, also the convergence time of $\hat{\mathbf{z}}_C$ can be shortened by acting on the value of $k_e$ in~\eqref{eq:cm_observer_eq}, and the gains of the consensus algorithms can be tuned in order to speed up the agreement~\cite{2007-OlfFaxMur}.

Additional extensive results are given in the Appendix.

\section{Conclusions}\label{sec:concl}

In this paper, we propose two fully-distributed methods for the estimation of the parameters needed by a planar multi-agent system to collectively manipulate an unknown load, i.e., the kinematic and dynamic parameters, as well as the estimate of the kinematic state of the load, i.e., the velocity of the center of mass and its rotational rate.
The approaches are totally distributed and rely on the geometry of the rigid body kinematics, on the rigid body dynamics, on nonlinear observation theory, and on consensus strategies. They are based on a sequence of steps that leads to states in which all agent agrees on the estimated parameters, during such steps any motion control law can be used (apart from a single step in the first algorithm).
The only requirements are related to the communication network, which is only required to be connected, and to the capability of each agent to control the local force applied to the load, while measuring the velocity of the contact point. 
Extensive numerical simulations confirm the effectiveness of our approach and its robustness to measurement noise and system size. 
Future works will deal with the manipulation of 3D objects in the aereal domain, extending the aerial manipulation methods in~\cite{2014a-GioFraSalSchPra,2014b-GioRylPraBueFra,2014d-YueSecBueFra}.

\appendix

\subsection{Manipulation of an unknown load by a team of 10 agents}

Figures.~\ref{fig:estimation_fullyConn} and~\ref{fig:estimation_fullyConn_noNoise} illustrate the same simulation setup of Sec.~\ref{sec:simulations_A} in the case of a fully connected topology and two different levels of noise. Comparing the time sequence $t_1\ldots t_6$,  we observe that a more connected topology and a lower noise level are factors that sensibly reduce the execution time of each step.-

\begin{figure}[t]
\centering
\includegraphics[width=0.45\textwidth]{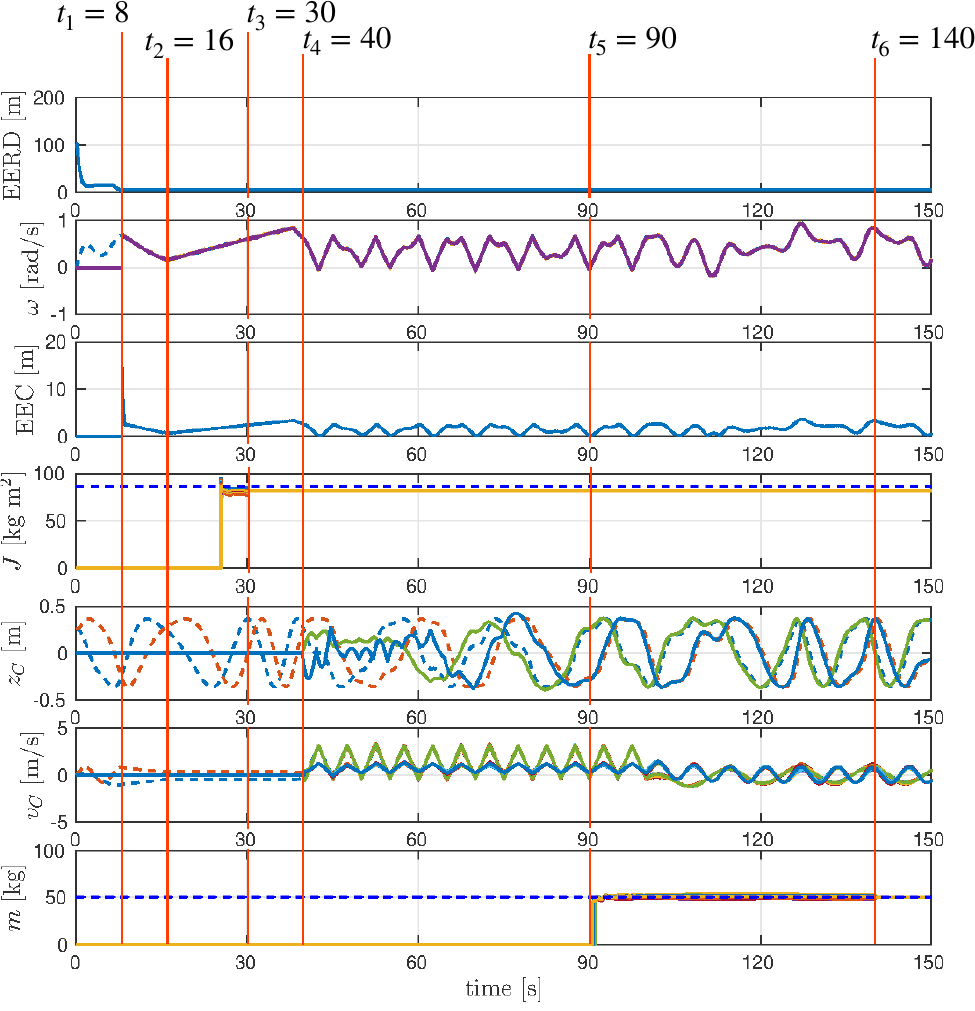}
\caption{Simulation with the same setup as in Fig.~\ref{fig:estimation}, but with a fully connected topology and noise $\sigma = 0.3$ m/s. }
\label{fig:estimation_fullyConn}
\end{figure}
\begin{figure}[t]
\centering
\includegraphics[width=0.45\textwidth]{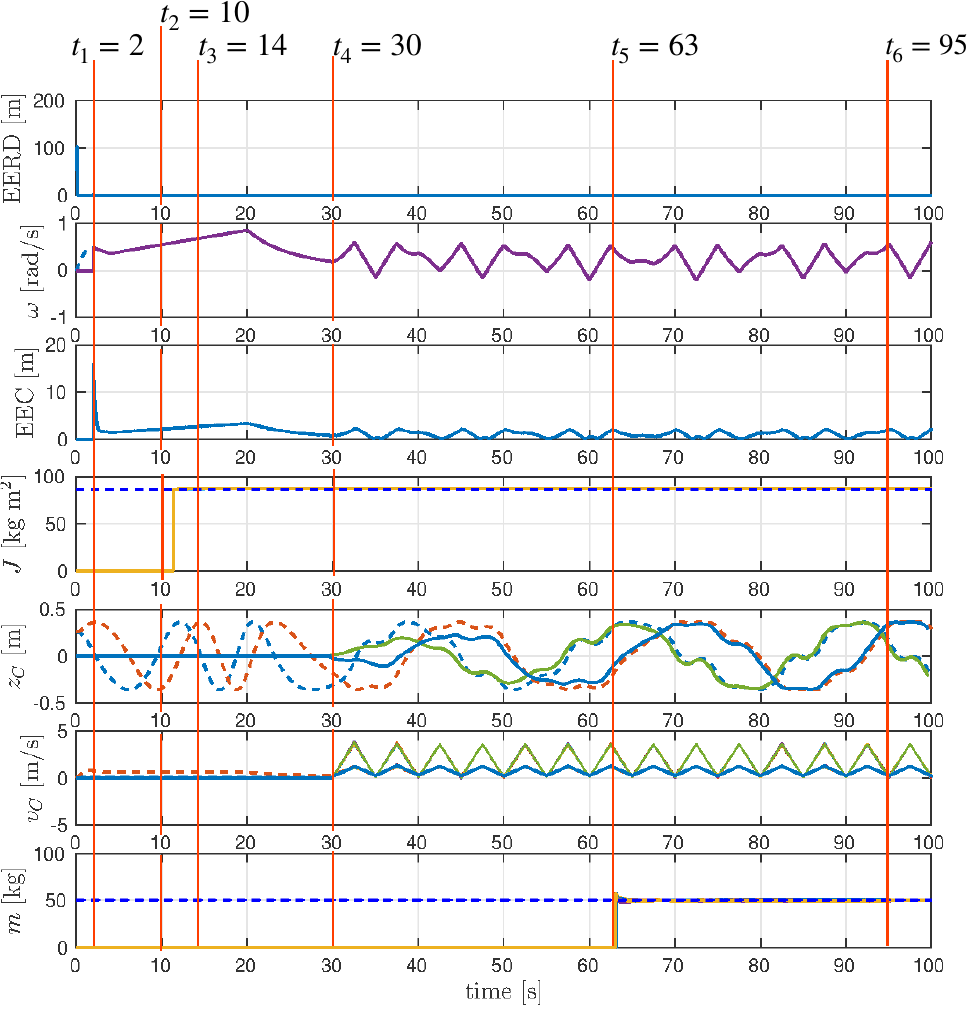}
\caption{Simulation with the same setup as in Fig.~\ref{fig:estimation}, but with a fully connected topology and noise $\sigma = 0.0001$ m/s.}
\label{fig:estimation_fullyConn_noNoise}
\end{figure}

\subsection{Performance Assessment and Uncertainty Propagation}\label{sec:perfAsses_n_uncertProp}

\begin{figure}[t]
\centering
\includegraphics[width=0.85\columnwidth]{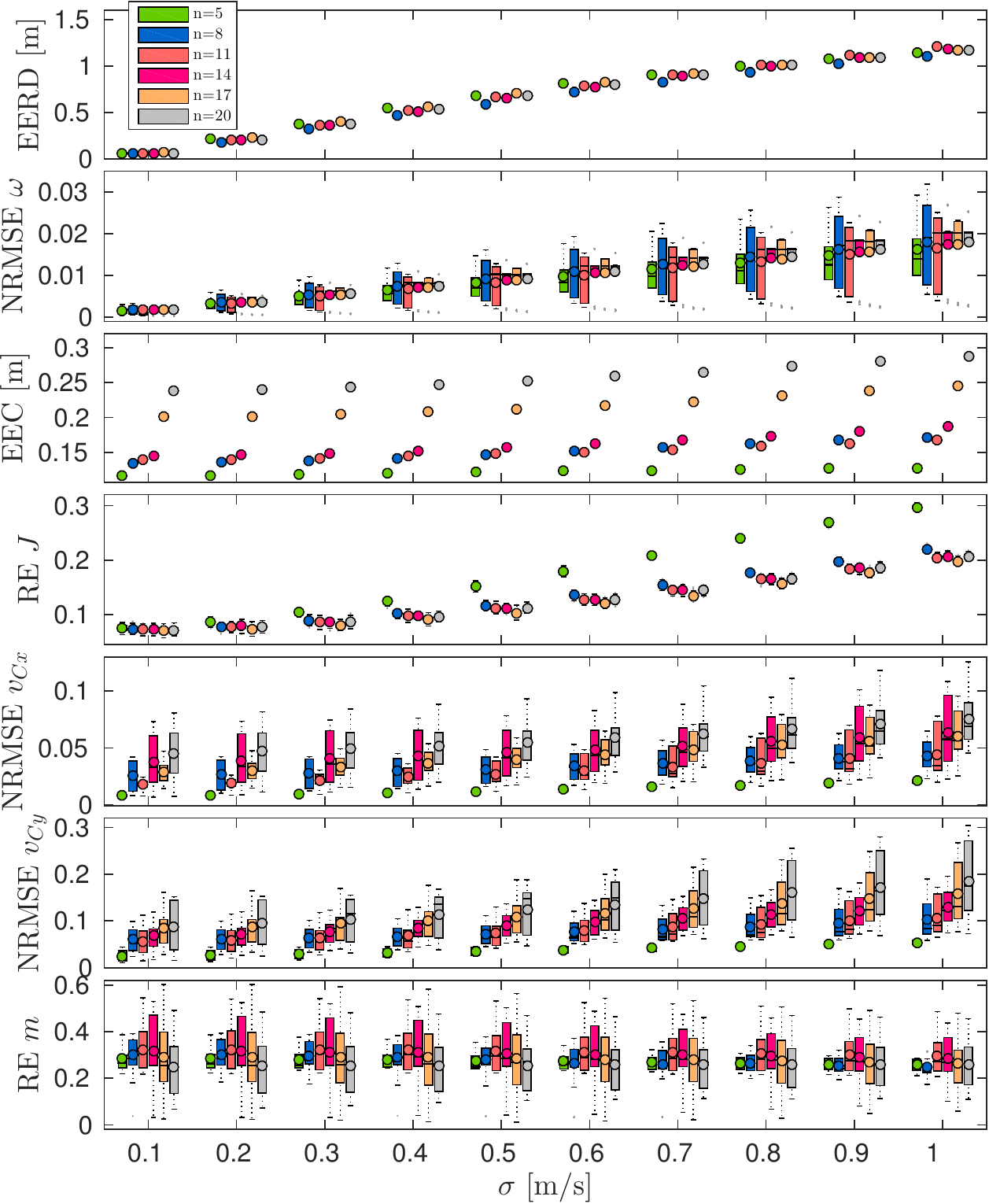}
\caption{
Box-plots obtained by running the whole algorithm in different operational conditions, selected by setting the pair $(n, \sigma)$. 
The bottom of a box indicates the first quartile, while the top indicates the third quartile. The dot inside a box is the second quartile, i.e., the median. The whisker at the bottom/top of a box indicates the lowest/highest data within 1.5 of the Interquartile Range of the lower/upper quartile.
}
\label{fig:mc_eval}
\end{figure}
\begin{figure}[th]
\centering
\includegraphics[width=0.9\columnwidth]{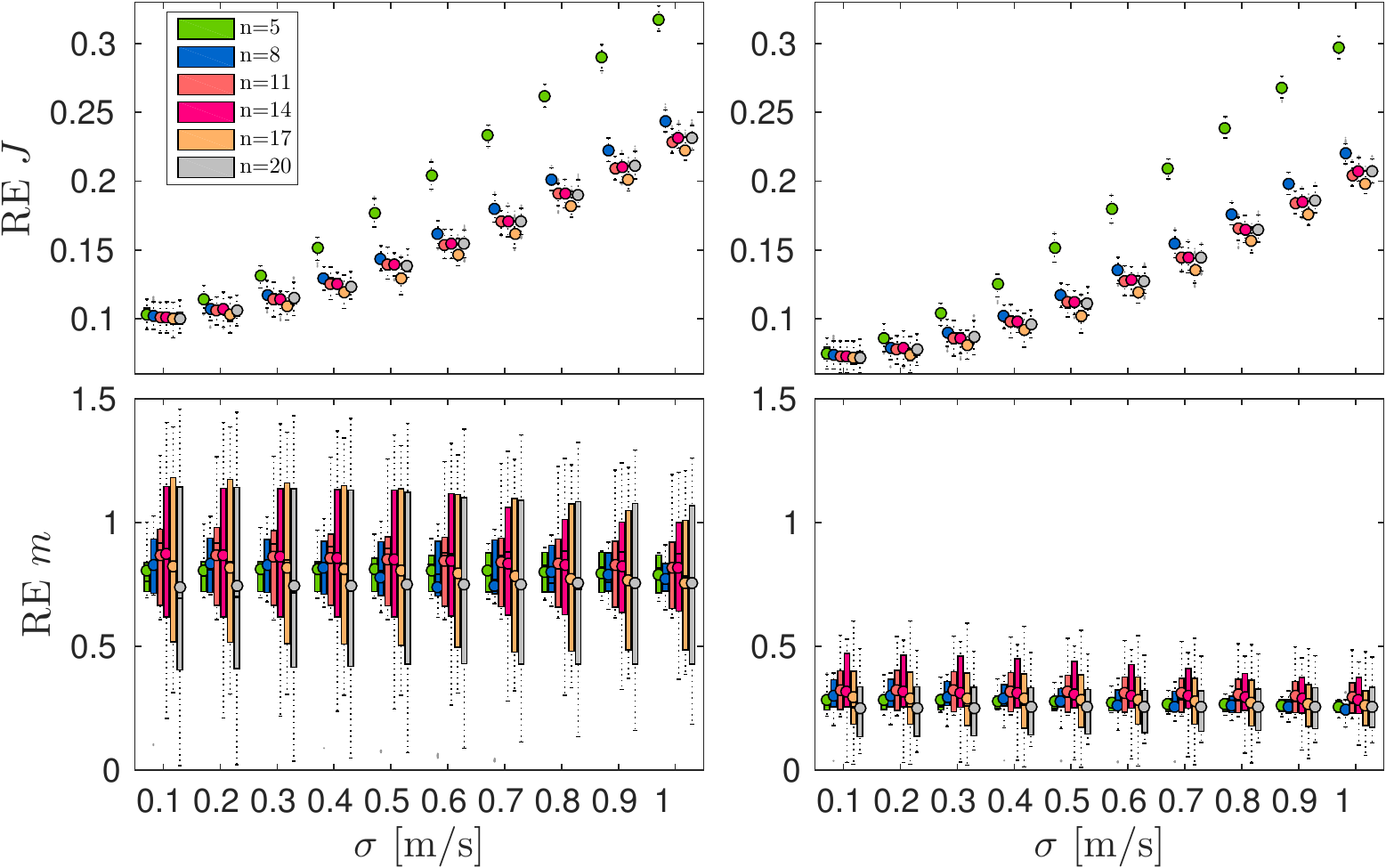}
\caption{
Box-plots of the Relative Error (RE) of the estimation of $m$ and $J$ obtained by running the whole algorithm in different conditions described by the pairs $(n, \sigma)$. The left and right column show the RE's before and after running the average consensus, respectively.
}
\label{fig:cons_eval}
\end{figure}
We analyze the performance of the algorithm by running a Monte Carlo simulation campaign.  
Specifically, a wide range of operational conditions is considered, defined by the pair $(n, \sigma)$ where $n$ indicates the number of agents in the network %
 and $\sigma$ the standard deviation defining the covariance of a zero-mean Gaussian noise, added to velocity measurements, as $\boldsymbol\Sigma_{i} = \sigma^{2} \mathbf{I}_{2\times 2}$, for each agent $i=1,\dots, n$.
 
The two parameters assume values over a 2D grid formed with the values $n \in \{5, 8, 11, 14, 17, 20 \}$ and $\sigma\in\{0.1, 0.2, 0.3, \dots, 1 \}$ m/s. Fifty independent simulations are run for each parameter pair. 
To ensure  consistency of comparison in each simulation  the mass  is set as $m=5\,n$\,kg (and the inertia $J$ is computed accordingly).
The agents communicate over the worst-case line-topology.

Figure~\ref{fig:mc_eval} illustrates the results, where each box-plot corresponds to the $50$ independent simulations executed for a given parameter pair.  
For consistency of comparison across the changing $n$, we use the 
Normalized Root Mean Square Error (NRMSE) and the Relative Error (RE), respectively for time-varying and constant parameters. 
The results show that the estimation accuracy decreases with an increasing noise level $\sigma$ for the EERD, the EEC, the estimation of $\omega$, the estimation of $\mathbf{v}_{C}$, and the estimation of $J$. On the contrary, the degradation in the estimation accuracy is of minor importance in the estimation of $m$. We observe that the error in the estimation of $J$ is strongly dependent on $n$: fixing a value for $\sigma$, the estimation accuracy increases with an increase in the number of agents. This is related to the use of the consensus algorithm for averaging out the noise, as will be better shown later in this section.
On the other hand, fixing the value for $\sigma$, the dispersion of the RE in the estimation of $m$ increases with increasing $n$. This behavior is related to 
the greater number of noisy measurements used to estimate $\mathbf{v}_{C}$, used in turn to estimate $m$. The degree of dispersion of the estimation of $\omega$ is clearly influenced by the noise level, the dispersion is close to zero for the EERD, EEC, and for the error in the estimation of $J$. Finally, fixing the number of agents $n$, the degree of the dispersion is constant with respect to the noise variation for the estimation error of $\mathbf{v}_{C}$ and $m$, confirming that the estimation of $\mathbf{v}_{C}$ and $m$ is strongly influenced by the level of noise.

As already mentioned in Sec.~\ref{sec:est_J}, for the estimation of constant parameters, such as $J$ and $m$, an average consensus algorithm is executed to average out the noise and improve the estimate. 
Figure~\ref{fig:cons_eval} shows the box-plots of the RE of the estimation of $J$ and $m$, before (on the left side) and after (on the right side)  the average consensus run. We observe that the application of the average consensus algorithm decreases significantly the average estimation error and also the dispersion.

\subsection{Inertia Moment and Mass Changing along the Task}\label{sec:numResTimeVar}
To show the ability of the second algorithm to deal with changing $m$ and $J$ we simulate a planar load with $m=50$\,kg and $J=86.89$\,kg\,m$^{2}$, manipulated by a team of $n=10$ agents communicating over a line-topology network and implementing  the observers introduced in  Sec.~\ref{sec:TimeVarInertParams}. A  zero-mean Gaussian noise with covariance matrix $\boldsymbol\Sigma_{i} = \sigma^{2} \mathbf{I}_{2\times 2}$ and $\sigma = 0.3$ m/s is added to the velocity measurements . After $100$s, we simulate a step-like decrease of mass $m$ and, consequently, a decrease of the moment of inertia $J$. 
Simulation results are illustrated in Fig.~\ref{fig:var_mass_n_inertia}. 
As theoretically proven, both the observers converge to a bounded region around the new true values. 
\begin{figure}[t]
\centering
\includegraphics[width=0.35\textwidth]{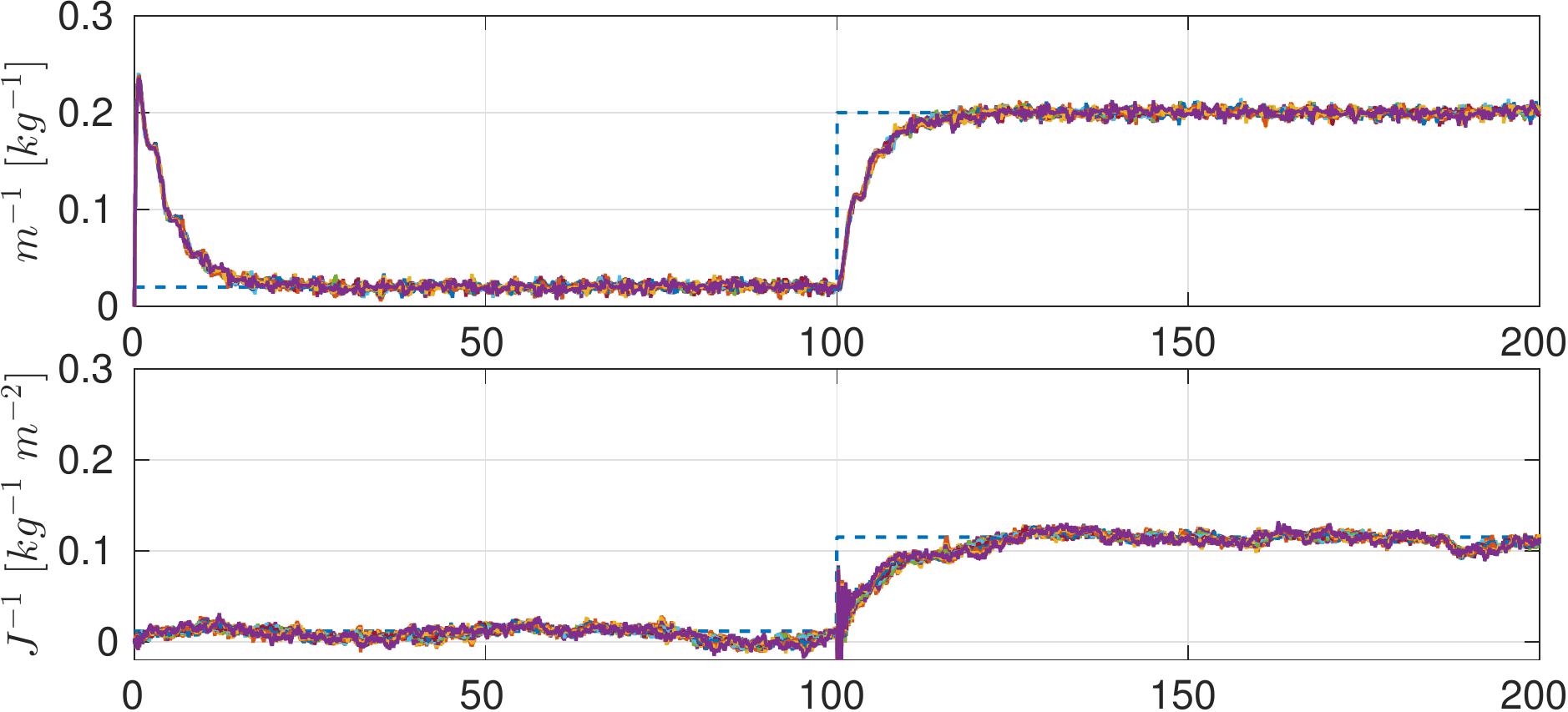}
\caption{
Simulation of the performance of the observer for the time-varying inertial parameters described in Sec.~\ref{sec:TimeVarInertParams}. To:  estimates of $m^{-1}$. Bottom: the estimates of $J^{-1}$. Dashed lines indicate true values, while solid lines indicate estimates.
}
\label{fig:var_mass_n_inertia}
\end{figure}

\bibliographystyle{IEEEtran}
\bibliography{./alias,./main}

\begin{IEEEbiography}[{\includegraphics[width=1in,height=1.25in,clip,keepaspectratio]{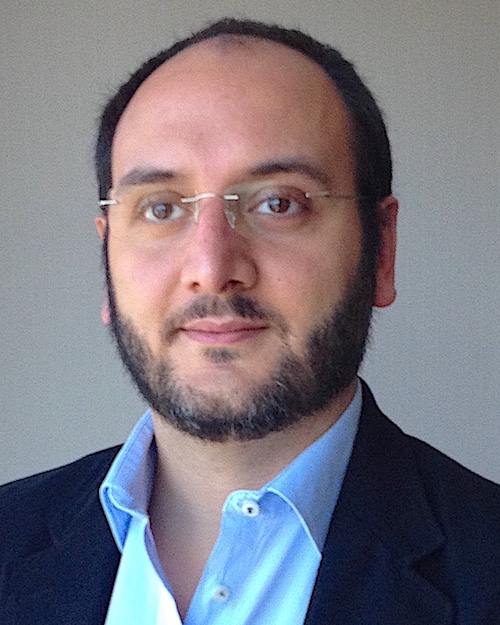}}]{Antonio Franchi} (S'07-M'11-SM'16) received the Ph.D. degree in system engineering from Sapienza University of Rome, Rome, Italy, in 2010 and the Habilitation to Direct Research (HDR) in Sciences from National Polytechnic Institute of Toulouse, Toulouse, France, in 2016.
In 2009, he was a Visiting Scholar with University of California at Santa Barbara, Santa Barbara, CA, USA. From 2010 to 2014, he was Research Scientist, Senior Research Scientist, and the Project Leader of the Autonomous Robotics and Human Machine Systems Group, Max Planck Institute for Biological Cybernetics in T\"ubingen, Germany. Since 2014, he has been a Tenured CNRS Researcher with the RIS team, LAAS-CNRS, Toulouse, France. He published more than 120 papers in peer-reviewed international journals and conferences. His main research interests include robotic systems, with a special regard to control of for aerial robots and multiple-robot systems.
Dr. Franchi received the IEEE RAS ICYA Best Paper Award in 2010. He is an Associate Editor for IEEE TRANSACTIONS ON ROBOTICS. He is the co-founder of the IEEE RAS Technical Committee on Multiple Robot Systems and of the International Symposium on Multi-Robot and Multi-Agent Systems.
\end{IEEEbiography}

\begin{IEEEbiography}[{\includegraphics[width=1in,height=1.25in,clip,keepaspectratio]{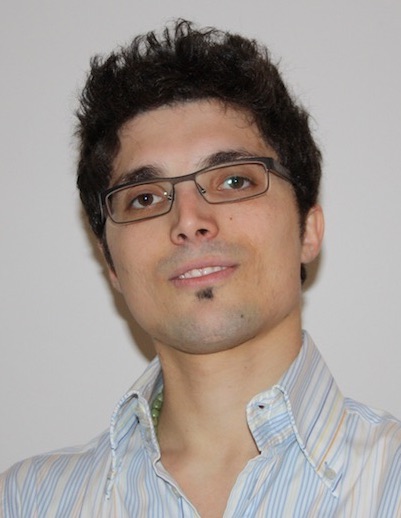}}]{Antonio Petitti} 
 received the B.S. and M.S. (Laurea Specialistica) degrees (summa cum laude) in Automation Engineering from Politecnico di Bari, Italy, in 2008 and 2010, respectively. From June 2011 to May 2018, he was Research Assistant at the Institute of Intelligent Systems for Automation (ISSIA) of the National Research Council (CNR), Italy. In 2015, he received the Ph.D. degree in Electrical and Information Engineering at Politecnico di Bari, Italy, and the joint Ph.D. degree of high qualification Scuola Interpolitecnica di Dottorato in Information and Communication Technologies. In 2013 and 2014 he was Visiting Research Fellow at ARHMS group, Max Planck Institute for Biological Cybernetics, T\"ubingen, Germany, and at RIS group LAAS-CNRS, Toulouse, France, respectively. Since June 2018, he has been Researcher at the Institute of Intelligent Industrial Technologies and Systems for Advanced Manufacturing (STIIMA) of the CNR, Bari, Italy. 
His scientific interests are focused on consensus theory and applications, distributed estimation, modeling and control of robotic networks.
\end{IEEEbiography}
\begin{IEEEbiography}[{\includegraphics[width=1in,height=1.25in,clip,keepaspectratio]{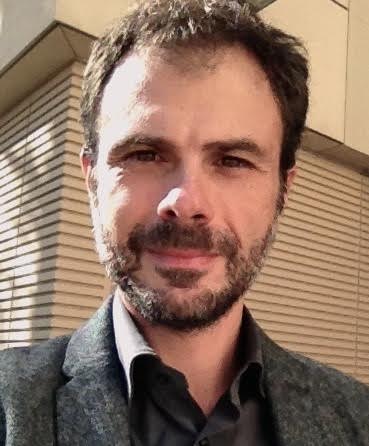}}]{Alessandro Rizzo} 
received the Laurea degree (summa cum laude) in computer engineering and the Ph.D. degree in automation and electronics engineering from the University of Catania, Italy, in 1996, and 2000, respectively. 
He is an Associate Professor at Politecnico di Torino, Italy, where he  is engaged in conducting and supervising research on cooperative robotics, complex networks and systems, modeling and control of nonlinear systems. Since 2012, he has also been a Visiting Professor at the New York University Tandon School of Engineering, Brooklyn, NY, USA. He is the author of two books, two international patents, and more than 130 papers on international journals and conference proceedings. Prof. Rizzo has been the recipient of the award for the best application paper at the IFAC world triennial conference in 2002 and of the award for the most read papers in Mathematics and Computers in Simulation (Elsevier) in 2009. Prof. Rizzo is also a Distinguished Lecture of the IEEE Nuclear and Plasma Science Society. More details can be found at the website \url{staff.polito.it/alessandro.rizzo}.
\end{IEEEbiography}

\clearpage

\begin{center}
{\LARGE Proofs of Propositions\\~\\ \normalsize \rm Technical report associated with the paper:\\ \large``Distributed Estimation of State and Parameters in Multi-Agent   Cooperative Load Manipulation''\\
\normalsize \emph{IEEE Transactions on Control of Network Systems}}

{Antonio Franchi, Antonio Petitti, Alessandro Rizzo}

\end{center}
%

\begin{abstract}
	This document is a technical attachment to~\cite{2018q-FraPetRiz}.
	Here, we present the proofs of the Propositions.
\end{abstract}

\section{How to Cite this Work}

This technical report is accompanying the IEEE Transactions on Control of Network Systems paper~\cite{2018q-FraPetRiz}. If you wish to reference this work, please cite this paper as follows:

{\small
\begin{lstlisting}
@Article{Franchi19tcns,
  author  = {A. Franchi and A. Petitti and
             A. Rizzo},
  title   = {Distributed Estimation of State 
             and Parameters in Multi-Agent
             Cooperative Load Manipulation},
  journal = {{IEEE} Transactions on Control
             of Network Systems},
  year	  = {TBD},
  doi	  = {TBD},
}
\end{lstlisting}
}

\begin{proof}[Proof {of Proposition~\ref{prop:iss} in~\cite{2018q-FraPetRiz}.} \nopunct]
In the ideal case with $\varepsilon_{1,2,3}=0$, the origin of the system in~\eqref{eq:error_observer_eq} is asymptotically stable (Theorem V.1,~\cite{2015b-FraPetRiz}). 
Consider the following Lyapunov function $V(\mathbf{e}) = \frac{1}{2} \mathbf{e}^\top \mathbf{e}$, then
\begin{equation}
\begin{aligned}
\dot{V} =& -e_1 e_2 x_3 + e_1 \widetilde{u}_2 e_3 + e_2 e_1 x_3 - e_2 \widetilde{u}_1 e_3 + \dots\\
& e_3 e_1 \widetilde{u}_2 - e_3 e_2 \widetilde{u}_1 - k_e e^2_3 - e_3 \varepsilon_3\\
& = - k_e e^2_3 - e_3 \varepsilon_3 \leq - k_e e^2_3 + |e_3| |\varepsilon_3|.
\end{aligned}
\end{equation}
We observe that
\begin{equation}
\begin{aligned}
- k_e e^2_3 +& |e_3| |\varepsilon_3| =\\
= & - k_e e^2_3 + |e_3| |\varepsilon_3| + k_e e^2_1 - k_e e^2_1 + k_e e^2_2 - k_e e^2_2\\
= & - k_e \|\mathbf{e}\|^2 + |e_3| |\varepsilon_3| + k_e e^2_1 + k_e e^2_2.
\end{aligned}
\end{equation}
Considering that $\|\mathbf{e}\| \geq |e_3|$ and $\|\boldsymbol{\varepsilon}\| \geq |\varepsilon_3|$, we have that
\begin{equation}
\begin{aligned}
\dot{V} \leq& - k_e \|\mathbf{e}\|^2 + |e_3| |\varepsilon_3| + k_e e^2_1 + k_e e^2_2\\
 \leq& - k_e \|\mathbf{e}\|^2 + \|\mathbf{e}\| \|\boldsymbol{\varepsilon}\| + k_e e^2_1 + k_e e^2_2.
\end{aligned}
\end{equation}
Moreover, considering that $k_e \|\mathbf{e}\|^2 \geq k_e e^2_1 + k_e e^2_2 $, we can write
\begin{equation}
\begin{aligned}
\dot{V} \leq& - k_e \|\mathbf{e}\|^2 + \|\mathbf{e}\| \|\boldsymbol{\varepsilon}\| + k_e e^2_1 + k_e e^2_2\\
\leq& - k_e \|\mathbf{e}\|^2 + \|\mathbf{e}\| \|\boldsymbol{\varepsilon}\| + k_e \|\mathbf{e}\|^2.
\end{aligned}
\end{equation}
Thus, we rewrite the foregoing inequality as
\begin{equation}
\dot{V} \leq - k_e (1-\theta) \|\mathbf{e}\|^2 - k_e \theta \|\mathbf{e}\|^2 + \|\mathbf{e}\| \|\boldsymbol{\varepsilon}\| + k_e \|\mathbf{e}\|^2,
\end{equation}
where $0 < \theta < 1$. The inequality $- k_e \theta \|\mathbf{e}\|^2 + \|\mathbf{e}\| \|\boldsymbol{\varepsilon}\| + k_e \|\mathbf{e}\|^2\leq 0$ holds if $\|\mathbf{e}\| \geq \frac{\|\boldsymbol{\varepsilon}\|}{k_e (\theta - 1)}$.
Thus, 
\begin{equation}
\label{eq:iss_condition}
\dot{V} \leq - k_e (1-\theta)\|\mathbf{e}\|^2,\,\, \forall \|\mathbf{e}\| \geq \frac{\|\boldsymbol{\varepsilon}\|}{k_e (\theta - 1)}.
\end{equation}
Hence, the system is ISS {(Theorem 4.19,~\cite{2002-Kha})}.
\end{proof}

\begin{proof}[Proof {of Proposition~\ref{prop:observer_J} in~\cite{2018q-FraPetRiz}.} \nopunct]
Define the error vector as $\mathbf{e} = \left[ e_1 \ e_2 \ e_3 \ e_4 \right]^\top = \left[(x_1x_4-\hat x_1)\ (x_2x_4-\hat x_2)\ (x_3-\hat x_3) \ (x_4-\hat x_4)\right]^\top$. After some algebra, the error dynamics is given by
\begin{equation}
\label{eq:lin_sys_error_cm_estim}
\dot{\mathbf{e}} = 
\left[\begin{matrix}
0 & -x_3 & - u_2 & 0\\
x_3 & 0 & u_1 & 0 \\
u_2 & - u_1 & -k_e & -u_3\\
0 & 0 & u_3 & 0\\
\end{matrix}
\right]
 \mathbf{e} = \left[\mathbf{U} + {\rm diag}(0,0,-k_e,0)\right]\mathbf{e},
\end{equation}
where $\mathbf{U}$ is skew symmetric, i.e., $\mathbf{U}+\mathbf{U}^\top=\mathbf{0}$.
%
Define the following candidate Lyapunov function:
$
V(\mathbf{e}) = \frac{1}{2}\mathbf{e}^\top\mathbf{e},
$
whose time derivative along the system trajectories is
\begin{equation}
\dot V = \mathbf{e}^\top\dot{\mathbf{e}} = \mathbf{e}^\top\mathbf{U}\mathbf{e} -k_e e_3^2 = -k_e e_3^2,
\label{eq:lyap_fun_cm_estim}
\end{equation}
which is negative semidefinite.
Now in order to study the invariant set that ensures that $\dot V = 0$ we impose that $e_3\equiv 0$, which implies, in particular, that $e_3= 0$, $\dot e_3= 0$, and $\ddot e_3= 0$. Considering, for simplicity, the case in which inputs are stepwise constant, the last three equations \eqref{eq:cm_observer_eq_tvar}--\eqref{eq:lyap_fun_cm_estim} result in the following system of linear equations:
\begin{align}
\left[
\begin{matrix} 
u_2 & - u_1 & u_3\\
-x_3 u_1 & -x_3 u_2 & 0\\
-x_3^2 u_2 & x_3^2 u_1 & 0\\
\end{matrix}
\right]
\left[
\begin{matrix}
e_1\\
e_2\\
e_4
\end{matrix}
\right]
=
\mathbf{E}
\left[
\begin{matrix}
e_1\\
e_2\\
e_4
\end{matrix}
\right]
=
\left[ 
\begin{matrix}
0\\
0\\
0
\end{matrix}
\right].
\end{align}
The determinant of $\mathbf{E}$ is $-u_3x_3^3(u_1^2+u_2^3)$. If the assumptions of the theorem hold, then $\mathbf{E}$ is nonsingular and therefore the only trajectory of the system that ensures $\dot V=0$ is $\mathbf{e}=\mathbf{0}$.
\end{proof}

\begin{proof}[Proof {of Proposition~\ref{prop:observer_m} in~\cite{2018q-FraPetRiz}.} \nopunct]
Define the error vector $\mathbf{e} = (e_1\,\,\, e_2)^\top = (z_1- \hat{z}_1\,\,\, z_2- \hat{z}_2)^\top$. The error dynamics is given by
\begin{equation}
\label{eq:lin_sys_error_m_estim}
\dot{\mathbf{e}} = 
\left[\begin{matrix}
-k_1 & u \\
-k_2 & 0 \\
\end{matrix}
\right]
 \mathbf{e} = \mathbf{A}\mathbf{e}.
\end{equation}
Consider the Lyapunov candidate $V = \frac{1}{2}\mathbf{e}^\top\mathbf{e}$, its time derivative is $\dot{V} = \mathbf{e}^\top\mathbf{A}\mathbf{e}$. If the eigenvalues of $\mathbf{A}$ are real and non-greater than $-\epsilon < 0$, then asymptotic stability is guaranteed. 
After some algebra, the eigenvalues of $\mathbf{A}$ are 
$$\lambda_{1,2} = \frac{1}{2}\left( -k_1 \pm \sqrt{k_1^2 - 4 k_2 u} \right).$$ Therefore, $\lambda_{1,2}$ are real iff $k_1^2 > 4k_2u$. Moreover, if $k_1 > \epsilon_1 >0$ and $k_2u > \epsilon_2 >0$, then $\exists \epsilon>0$ s.t. $\lambda_{1,2}<-\epsilon < 0$.
\end{proof}

\begin{proof}[Proof {of Proposition~\ref{prop:bounded_nonvanishing_omega} in~\cite{2018q-FraPetRiz}.} \nopunct]
In order to prove~\eqref{eq:omega_bound}, let us consider the quantity $\alpha = \omega^2 - 2 n J^{-1}  \textbf{z}_C^\top \textbf{f}^*$. Let us now take the derivative of $\alpha$ w.r.t. time. Using~\eqref{eq:omega_dyn_z_C} and~\eqref{eq:zC_dyn}, we obtain $\dot \alpha = 0$,
i.e., $\alpha$ is an invariant along the system trajectories when $\mathbf{f}_i = \mathbf{f}^*= {\rm const}$, $\forall i=1 \ldots n$. In particular, $\alpha(t) = \alpha(0)$, which implies
\begin{align}
\omega^2(t)  &= \omega^2(0) - 2 n J^{-1}  \textbf{z}_C^\top(0) \textbf{f}^* + 2 n J^{-1}  \textbf{z}_C^\top(t) \textbf{f}^* \label{eq:equal_omega}\\
&= \omega^2(0) +  2 n J^{-1} ( \textbf{z}_C(t) -  \textbf{z}_C(0))^\top  \textbf{f}^* \notag\\
&\le \omega^2(0) +  2 n J^{-1} \| \textbf{z}_C(t) -  \textbf{z}_C(0)\|  \|\textbf{f}^*\| \notag\\
&\le \omega^2(0) +  4 n J^{-1} \| \textbf{z}_C\| \| \textbf{f}^* \|, 
\label{eq:ineq_omega}
\end{align}
which in turn proves~\eqref{eq:omega_bound}. Note that we used the fact that $\| \textbf{z}_C\|$ is constant over time to derive~\eqref{eq:ineq_omega}. 

In order to prove~\eqref{eq:critical_init_cond}, we impose that $\omega({\bar t})$ is identically zero, along with all its derivatives. 
Imposing in~\eqref{eq:equal_omega} that $\omega({\bar t})=0$, we obtain 
\begin{align}
0  = \omega^2(0) - 2 n J^{-1}  \textbf{z}_C^\top(0) \textbf{f}^* + 2 n J^{-1}  \textbf{z}_C^\top({\bar t}) \textbf{f}^*.
\label{eq:invariant_at_T}
\end{align} 
Setting $\dot \omega({\bar t}) = 0$,  $\mathbf{f}_i=\mathbf{f}^*$, $\tau_i=0$, $\forall i=1 \ldots n$ in~\eqref{eq:omegadot}, we obtain
\begin{align}
{\textbf{z}_C^\perp({\bar t})}^\top\mathbf{f}^* = 0,
\quad
\Rightarrow
\quad
{\textbf{z}_C^\top({\bar t})}\mathbf{f}^* = \|\textbf{z}_C\|\|\mathbf{f}^*\|.
\label{eq:z_C_f_start_prodoct}
\end{align}
Plugging~\eqref{eq:z_C_f_start_prodoct} in~\eqref{eq:invariant_at_T} gives
\begin{align}
0  = \omega^2(0) - 2 n J^{-1}  \textbf{z}_C^\top(0) \textbf{f}^* + 2 n J^{-1}  \|\textbf{z}_C\|\|\mathbf{f}^*\|, \label{eq:invariant_at_T_II}
\end{align} 
which, reordered, gives~\eqref{eq:critical_init_cond}. The proof is concluded by noticing that $\omega({\bar t})=\dot \omega({\bar t})=0$ implies (see~\eqref{eq:omega_dyn_z_C} and~\eqref{eq:zC_dyn}) that all the higher order derivatives of $\omega$ at ${\bar t}$ are zero as well.
\end{proof}

\begin{proof}[Proof {of Proposition~\ref{prop:stopping_motion} in \cite{2018q-FraPetRiz}.} \nopunct]
From~\eqref{eq:velocity_com} and using the identities on the left-hand side in~\eqref{eq:zero_ident}, it is straighforward to derive the following two identities:
\begin{align}
\sum_{i=1}^n\mathbf{v}_{C_i} = n\mathbf{v}_{C} + n\omega\mathbf{z}_{C}^\perp\;, 
\quad 
\mathbf{v}_{C_i} = \mathbf{v}_{C} + \omega(\mathbf{z}_{C} + \mathbf{z}_i)^\perp,
\notag
\end{align}
which can be used to obtain, respectively,
\begin{align}
&\quad \quad \quad n\mathbf{f}_{\rm mean} = -b \sum_{i=1}^n \mathbf{v}_{C_i} = -bn\mathbf{v}_{C} -nb\omega\mathbf{z}_{C}^\perp \quad
\label{eq:fstar_damped}\\
\end{align}
and
\begin{align}
\eta &= -\frac{b}{J}\sum_{i=1}^n\mathbf{z}_i^{\perp^\top}\mathbf{v}_{C_i} = 
-\frac{b}{J}
\Bigg(
(\mathbf{v}_{C} + \omega\mathbf{z}_{C}^\perp )^\top
\underbrace{\sum_{i=1}^n\mathbf{z}_i^\perp}_{=0} 
+
\omega\sum_{i=1}^n\mathbf{z}_i^{\perp^\top} \mathbf{z}_i^\perp
\Bigg)\notag\\
&=
-\frac{b}{J}\omega\sum_{i=1}^n\|\mathbf{z}_i\|^2.\label{eq:eta_damped}
\end{align}
%
Plugging~\eqref{eq:fstar_damped} and~\eqref{eq:eta_damped} in~\eqref{eq:omega_dyn_z_C} and~\eqref{eq:lls_mass} we obtain
\begin{align}
J \dot \omega &= 
-bn
\left(
{\mathbf{z}_C^\perp}^\top \mathbf{v}_{C} 
+
\omega{\mathbf{z}_C^\perp}^\top \mathbf{z}_{C}^\perp
\right)
 -b\omega\sum_{i=1}^n\|\mathbf{z}_i\|^2
\notag\\
m \dot{\mathbf{v}}_C &= -bn
\left(
\mathbf{v}_{C} + \omega\mathbf{z}_{C}^\perp
\right).
\notag
\end{align}
Let us consider  
$V=\tfrac{J\omega^2+m\|\mathbf{v}_C\|^2}{2}$ as a Lyapunov candidate function.
We obtain
\begin{align}
\dot V &= 
-bn
\left(
\mathbf{v}_{C}^\top\mathbf{v}_{C}
+
{2\omega\mathbf{z}_C^\perp}^\top \mathbf{v}_{C} 
+
\omega^2{\mathbf{z}_C^\perp}^\top \mathbf{z}_{C}^\perp
\right)
 -b\omega^2\sum_{i=1}^n\|\mathbf{z}_i\|^2
= \notag\\
& =  -b  
\left(
n\|
\mathbf{v}_{C}
+
\omega\mathbf{z}_C^\perp
\|^2
+
\omega^2\sum_{i=1}^n\|\mathbf{z}_i\|^2 
\right) < 0 \quad \forall [\mathbf{v}_{C}^\top \,\omega]\neq \mathbf{0}^\top,
\notag
\end{align}
which proves the thesis of the proposition.
\end{proof}

\end{document}